%% file: main.tex
\documentclass[acmsmall,nonacm]{acmart}

\usepackage{padhi_paper}
\usepackage{newtxmath}

\input{config}
\input{macros}

\begin{document}

\title{On Scaling Data-Driven Loop Invariant Inference}

\author{Sahil Bhatia}
\affiliation{\institution{Microsoft Research, India}}
\email{t-sab@microsoft.com}
\author{Saswat Padhi}
\affiliation{\institution{University of California Los Angeles, USA}}
\email{padhi@cs.ucla.edu}
\author{Nagarajan Natarajan}
\affiliation{\institution{Microsoft Research, India}}
\email{nagarajn@microsoft.com}
\author{Rahul Sharma}
\affiliation{\institution{Microsoft Research, India}}
\email{rahsha@microsoft.com}
\author{Prateek Jain}
\affiliation{\institution{Microsoft Research, India}}
\email{prajain@microsoft.com}


\input{sections/abstract}
\maketitle

\input{sections/introduction}

\input{sections/example}
\input{sections/background}
\input{sections/overview}
\input{sections/learning}
\input{sections/evaluation}
\input{sections/related-work}
\input{sections/conclusion}

\bibliography{main}

\end{document}

%% file: config.tex
\newcommand{\CodePunct}[1]{\textcolor{gray!80!black}{\texttt{#1}}}

\lstdefinelanguage{SyGuS}{
  alsoletter={0, 1, +, -, *, =, <, >},
  literate={(}{{\CodePunct{(}}}1
           {)}{{\CodePunct{)}}}1,
  texcl=true, 
  morecomment=[l]{;},
  morestring=[b]",
  keywords=[1]{NIA, LIA},
  keywords=[2]{set-logic, define-fun, declare-var, check-synth, synth-fun, constraint},
  keywords=[3]{Constant, Variable},
  keywords=[4]{Int, Bool},
  keywords=[5]{0, 1, 2, 1000, true, false},
  keywords=[6]{+, -, *, =, <, >, >=, <=, div, mod},
  keywords=[7]{and, or, not, ite, =>}
}

%% file: macros.tex
\newcommand{\redacted}[1]{$\langle$\,\emph{redacted for double-blind reviewing}\,$\rangle$}

\newcommand{\tool}[1]{\textsc{#1}\xspace}
\newcommand{\oasis}{\tool{Oasis}}

\newcommand{\lig}{\tool{LoopInvGen}}

\newcommand{\cvc}{\tool{CVC4}}
\newcommand{\dsynth}{\tool{DradSynth}}
\newcommand{\oasiscl}{\tool{Oasis, Naive Vars Select}}
\newcommand{\oasiscomp}{\tool{Oasis, No Vars Select}}
\newcommand{\oasisdt}{\tool{Oasis+DT}}
\newcommand{\oasiscomplete}{\tool{Oasis, Complete Maps}}
\newcommand{\oasisact}{\tool{Oasis, No  \AlgoStyle{\RelInfer}}}
\newcommand{\oasisopt}{\tool{Oasis, No Optimization}}
\newcommand{\unconfounded}{Unconfounded}
\newcommand{\confoundedone}{Confounded1}
\newcommand{\confoundedfive}{Confounded5}
\newcommand{\sygus}{Sygus 2018}

\newcommand{\learn}{Learn}

\newcommand{\LearnClassifier}{\learn}

\newcommand{\FindPosCounterExample}{FindPosCounterExample}
\newcommand{\FindNegCounterExample}{FindNegCounterExample}
\newcommand{\RelInfer}{RelInfer}
\newcommand{\pseudo}[1]{\textbf{\textsf{#1}}}

\newcommand{\sname}[1]{\textsf{#1}\xspace}

\newcommand{\syguscomp}[1][]{\sname{SyGuS-Comp#1}}

\newcommand{\program}{\ensuremath{\textsf{\textbf{P}}}\xspace}
\newcommand{\vc}[1]{\ensuremath{\text{VC}_\textsf{#1}}}

\newcommand{\integer}{\ensuremath{\mathbb{Z}}}
\newcommand{\intloopinv}{\ensuremath{\mathcal{H}_\textsc{CNF}}}
\newcommand{\true}{\ensuremath{\textsf{True}}}
\newcommand{\false}{\ensuremath{\textsf{False}}}
\newcommand{\pstate}{\ensuremath{\sigma}}

\newcommand{\ind}[1]{\ensuremath{\vmathbb{1}\left\{#1\right\}}}
\newcommand{\ii}{\ensuremath{n}}
\newcommand{\ci}{\ensuremath{c}}
\newcommand{\di}{\ensuremath{d}}
\newcommand{\nii}{\ensuremath{N}}
\newcommand{\nc}{\ensuremath{C}}
\newcommand{\nd}{\ensuremath{D}}
\newcommand{\w}{\ensuremath{\mathbf{w}}}
\newcommand{\x}{\ensuremath{\mathbf{x}}}
\newcommand{\z}{\ensuremath{z}}
\newcommand{\ylor}{\ensuremath{y^\lor}}

\newcommand{\inv}{\ensuremath{\mathcal{I}}}
\newcommand{\pre}{\ensuremath{\rho}}
\newcommand{\post}{\ensuremath{\phi}}

\newcommand{\Classifier}{\ensuremath{\mathcal{C}}}

\newcommand{\PState}{\ensuremath{\vec{\sigma}}}
\newcommand{\Assign}[2]{\ensuremath{(#1 \mapsto #2)}}

\newcommand{\NegPStates}{\ensuremath{\bm{\PState_{-}}}}
\newcommand{\PosPStates}{\ensuremath{\bm{\PState_{+}}}}

\NewDocumentCommand
  {\Vars}
  { O{} }
  {\ensuremath{\vec{x}\mkern0.5mu_\text{\smaller#1}}}

\NewDocumentCommand
  {\InvRel}
  { d() }
  {\ensuremath{\mathcal{I}\IfValueT{#1}{(#1)}}}
\NewDocumentCommand
  {\PreRel}
  { O{} d() }
  {\ensuremath{\textsc{\large\sffamily pre}_\text{\smaller\kern0.1pc#1}\IfValueT{#2}{(#2)}}}
\NewDocumentCommand
  {\TransRel}
  { O{} d() }
  {\ensuremath{\textsc{\large\sffamily trans}_\text{\smaller\kern0.1pc#1}\IfValueT{#2}{(#2)}}}
\NewDocumentCommand
  {\PostRel}
  { O{} d() }
  {\ensuremath{\textsc{\large\sffamily post}_\text{\smaller#1}\IfValueT{#2}{(#2)}}}

\newcommand{\triple}[3]{\ensuremath{\left\{#1\right\}\ #2\ \left\{#3\right\}}}

\let\oldequiv\equiv
\renewcommand\equiv{\mathrel{\raisebox{0.5pt}{$\oldequiv$}}}
\let\oldimplies\implies
\renewcommand\implies{\mathrel{\resizebox{1.5em}{0.425em}{$\oldimplies$}}}

%% file: sections/abstract.tex
\begin{abstract}
    Automated synthesis of inductive invariants is an important problem in software verification.
    Once all the invariants have been specified, software verification reduces to checking of verification conditions.
    Although static analyses to infer invariants have been studied for over forty years,
    recent years have seen a flurry of data-driven invariant inference techniques
    which guess invariants from examples instead of analyzing program text.
    However, these techniques have been demonstrated to scale only to programs with a small number of variables.
    In this paper, we study these scalability issues and address them in our tool \oasis\
    that improves the scale of data-driven invariant inference and outperforms state-of-the-art systems on benchmarks from the invariant inference track of the Syntax Guided Synthesis competition.
\end{abstract}

%% file: sections/introduction.tex
\section{Introduction}%
\label{sec:introduction}

Inferring inductive invariants is one of the core problems of software verification.
Recently, there has been a flurry of data-driven invariant inference techniques~\citep{dig,vueq,iac,esop,geometric,ice,c2i,icedt,ddchc,pie,eusolver,selective,posthat} that learn invariants from examples.
These data-driven techniques
offer attractive features such as the ability to systematically generate {\em disjunctive} invariants.
Whereas the well-know static invariant inference techniques either fail to infer disjunctive invariants~\cite{cousot77,polyhedra,octagon,sting} or require a user-provided bound on the number of disjunctions~\cite{bagnara,invgen,ssd,popl07,pldi08}. 
At the heart of the data-driven techniques is an {\em active learning}~\cite{hanneke2009theoretical} loop: a {\em learner} guesses a candidate invariant from data and provides the candidate to a {\em teacher}. The teacher either validates that the candidate is a valid invariant or returns a counterexample. This example is added to the data and the process is repeated until the learner guesses a correct invariant. In this architecture, the more the number of program variables in the verification problem, the more the learner is likely to choose an incorrect candidate or take a long time to generate good candidates~\cite{piehe}. 
Hence, as we discuss in our evaluation in \Cref{sec:evaluation}, existing data-driven invariant inference techniques have been shown to be effective only for programs with a small number of variables.

For data-driven invariant inference to be applicable to verification of practical software, these scalability challenges must be addressed.
There are two main obstacles to scalability which are related to the number of program variables:
First, a program can have many variables and often only a small subset of these variables are {\em relevant} to the invariants. Intuitively, writing correct programs that require complicated invariants with many variables is hard for developers and prior works on invariant inference are also biased towards simple invariants~\cite{beautiful}.
In the absence of a technique that separates the relevant variables from the irrelevant, the learner can get bogged down by the irrelevant variables. 
In particular, invariant inference benchmarks in the Syntax Guided Synthesis (SyGuS) competition are provided as logic formulas where static {\em slicing}~\cite{slice} fails to remove the semantically irrelevant variables.
Second, data-driven techniques also rely on some form of enumeration to generate candidate predicates. Thus, a higher number of variables causes the enumerator to take a long time to reach pertinent candidates. For example, if the enumerator exhaustively generates expressions in increasing size~\cite{eusolver,piehe} then before enumerating expressions of size $s$, it must enumerate all expression of size $s-1$ over all variables.

To exemplify these scalability issues, we consider \lig~\cite{pie,piehe}, the state-of-the-art data-driven invariant inference tool that 
won in the invariant inference track of SyGuS competition held in 2017 and 2018. It uses exhaustive enumeration to synthesize Boolean {\em features} (simple predicates). Then, a Boolean function learner generates a candidate invariant which is a Boolean combination of the features. As the number of program variables increases, the scalability degrades because the enumerator must explore exponentially many features and the learner needs many examples to avoid generating candidates with irrelevant variables; with a large number of variables, the learner can {\em overfit} on the irrelevant variables to generate incorrect candidates that will be rejected by the teacher \cite{piehe}.

We explore addressing both the scalability issues, caused by enumeration and irrelevant variables, through machine learning (ML). In particular, we make the following two contributions. First, we describe a learner that can infer the relevant variables, thus ensuring that data-driven invariant inference is only applied to the simpler problem with a few or no irrelevant variables. Since the number of relevant variables is typically small, data-driven invariant inference can scale to such tasks better. Second, we show that exhaustive enumeration can be replaced by learners that are much more scalable.
Instead of a {\em generate-and-check} approach where an enumerator generates all possible candidate features eagerly~\cite{daikon,escher,pie}, we employ a more scalable {\em guess-and-check} approach where the learner intelligently guesses features from data. 

We have implemented these techniques in a tool 
\oasis{}\footnote{The name \oasis{} stands for \textbf{O}ptimization \textbf{A}nd \textbf{S}earch for \textbf{I}nvariant \textbf{S}ynthesis.}
that takes as input logic formulas which encode the verification of safety properties of programs over integer variables and outputs inductive invariants that are sufficient to prove the properties. To this end, \oasis employs new ML algorithms for the well-known {\em binary classification} problem: the learner's goal is to find a  {\em classifier} that {\em separates} positive and negative examples. The classifier is  a predicate that includes the positive examples and excludes the negative examples. 
In the context of invariant inference, an {\em example} is a program {\em state} that maps variables to integers.
\oasis makes the following contributions.

First, \oasis uses binary classification to infer relevant and irrelevant variables (\Cref{sec:overview.refinement}). It uses symbolic execution to generate {\em reachable} states (positive examples) and {\em bad} states (negative examples), which are backward reachable from states that violate the safety properties.
Then it finds a {\em sparse} classifier and we classify the variables occurring in the classifier as relevant. If a variable is absent from the classifier and it is possible to separate samples of reachable states from bad states without using the variable then it is likely to be irrelevant to the invariant. The sparsity requirement ensures that we keep the number of relevant variables to  minimal.

We remark that we need a custom learner for this task. Each state (reachable or bad) is a partial map from variables to integers. In particular, there are some variables that are not in the domain of the state. These variables are {\em don't care}, \ie, they can be assigned any value without affecting the {\em label} (positive or negative) of the state. The (partial) models generated by SMT solvers typically have don't cares. The well-known ML classifiers learn over total maps as opposed to partial maps. Although one can extend a partial map to a total map by setting the don't care variables to zero or randomly assigned values, these alternatives are undesirable as a partial map corresponds to an infinite number of possible total maps and supplanting it with any total map loses the information encoded in the partial map. Hence, we have designed a custom learner that  directly learns a classifier using partial maps and is not limited to total maps. After we obtain the set of relevant variables from the classifier, \oasis calls a modified version of \lig where the synthesis of  features is restricted to predicates over the relevant variables.

Second, \oasis uses a learner to synthesize Boolean features from data (\Cref{sec:overview.inference}). Internally, \lig breaks down the problem of invariant inference into many small binary classification tasks and uses Escher~\cite{escher} to find features that solve them. Specifically, Escher exhaustively enumerates all features in increasing size till it finds one that separates the positive examples from the negative examples in the small task. \oasis replaces Escher with a learner to find such features. Unlike traditional ML algorithms that have non-zero {\em error}, \ie, they fail to separate some positive examples from some negative examples, \lig requires the feature synthesizer to have zero error.

\oasis uses the same learner to solve both these problems, \ie, inferring relevant variables and inferring features. In particular, the learner of \oasis solves a non-standard ML problem: finding sparse classifiers with zero error in the presence of don't cares.  
We describe a novel learner that solves this problem (\Cref{sec:learning}). To the best of our knowledge, all prior works on data-driven invariant inference use learners that require total maps. We show how to encode the problem of finding such a  classifier  as an instance of integer linear programming (ILP) which minimizes an objective function subject to linear constraints. Although linear programming has previously been used to assist invariant inference~\cite{invgen}, our encoding is novel. Specifically, we show how to systematically encode domain-specific heuristics as objective functions or constraints for effective learning in the context of invariant inference. Heavily optimized ILP solvers are available as off-the-shelf tools and \oasis uses them to scale data-driven invariant inference.  

To demonstrate the scalability of \oasis in practice, we evaluate \oasis on over 400 benchmarks from the invariant (Inv) track of the SyGuS competion held in 2019~\cite{syguscomp19} (\Cref{sec:evaluation}). This benchmark set includes the new community provided programs that have a large number of irrelevant variables which test the scalability of invariant synthesis tools~\cite{code2inv}. Our evaluation shows that \oasis significantly improves the scalability of data-driven invariant inference on these benchmarks and solves 20\% more benchmarks than \lig, the state-of-the-art data-driven invariant inference tool. 
\oasis even outperforms 
state-of-the-art invariant inference tools that are based on very different techniques. It solves more benchmarks than deductive synthesis implemented in CVC4~\cite{cvc4,cegqi}
and cooperative synthesis of the recent work of {\sc DryadSynth}~\cite{dryadsynth} that combines enumerative and deductive synthesis.
Thus, our evaluation shows that \oasis significantly improves the state-of-the-art in data-driven invariant inference and  makes it  as scalable as deductive and cooperative techniques.
\oasis solves more benchmarks than these tools and also solves benchmarks that are beyond the reach of prior work. 

The rest of the paper is organized as follows. We provide an example to show the end-to-end working of \oasis{} (\Cref{sec:example}) and review the relevant background (\Cref{sec:background}).
We describe \oasis in detail (\Cref{sec:overview})
followed by the  ILP-based learner (\Cref{sec:learning}).
We evaluate \oasis in \Cref{sec:evaluation}, place it in the context of the landscape of invariant inference techniques in \Cref{sec:related}, and conclude with  directions for future work in \Cref{sec:conclusions}.

%% file: sections/example.tex
\section{Working Example}%
\label{sec:example}

We use a simple (contrived) benchmark to show the working of each component of \oasis. 
The goal is to synthesize an inductive invariant $I(\vec{x})$, where $\vec{x}=\langle i,j,k,n,y\rangle$ that satisfies the following {\em verification conditions} (VCs) expressed as Horn Clauses.
\begin{eqnarray*}
\mathit{Pre}(\vec{x})\Rightarrow I(\vec{x})\mbox{ with } \mathit{Pre}(\vec{x}) \triangleq i=j=0\wedge k\geq 0\wedge n\geq 0 \\
I(\vec{x})\wedge \mathit{Trans}(\vec{x},\vec{x}')\Rightarrow I(\vec{x}')\mbox{ with }\mathit{Trans}(\vec{x},\vec{x}')\triangleq i\leq n\wedge i'=i+1\wedge j'=j+1\wedge y'=i\times j\\
I(\vec{x})\Rightarrow \mathit{Post}(\vec{x})\mbox{ with } \mathit{Post}(\vec{x})\triangleq i\leq n\vee i+j+k\geq 2n\vee y\geq n^2
\end{eqnarray*}
These VCs encode the verification of the C-program in Figure~\ref{fig:motivating-example}. If there exists a predicate $I$ that satisifes the VCs then for all possible inputs the assertion can never be violated. A {\em state} for this example is a 5-tuple that maps $i,j,k,n,y$ to integers or don't cares (denoted by $\top$).

\input{figures/motivating-example}

The first step is the identification of irrelevant variables. \oasis generates reachable states, i.e., positive examples by computing satisfying assignments of $\mathit{Pre}(\vec{x})$ and $\mathit{Pre}(\vec{x})\wedge\mathit{Trans}(\vec{x},\vec{x'})$. For bad states, i.e., negative examples, \oasis computes satisfying assignments of $\neg\mathit{Post}(\vec{x})$ and 
$\neg\mathit{Post}(\vec{x}')\wedge \mathit{Trans}(\vec{x},\vec{x}')$. 
These satisfying assignments are obtained from off-the-shelf SMT solvers and result in Table~\ref{tab:initdata}.

\begin{minipage}{.4\textwidth}
\begin{table}[H]
\begin{tabular}{cccccl}
\toprule
$i$ & $j$ & $k$ & $n$ & $y$ & $\ell$ \\ 
\midrule
0 & 0 & 0 & 0 & $\top$ & 1\\
2 & 3 & 0 & 1 & 2 & 1\\
1 & -1 & 0 & 0 & -1 & 0\\
6 & 4 & 0 & 5 & 15 & 0\\
\bottomrule
\end{tabular}
\caption{Initial symbolic execution data.}
\label{tab:initdata}
\end{table}
\end{minipage}
\begin{minipage}{.5\textwidth}
\begin{table}[H]
\begin{tabular}{cccccl}
\toprule
$i$ & $j$ & $k$ & $n$ & $y$ & $\ell$ \\ 
\midrule
0 & 0 & -2 & -1 & 0 & 0\\
2 & 3 & -3 & 1 & 0 & 0\\
\bottomrule
\end{tabular}
\caption{Additional data from robustness checking.}
\label{tab:ssdata}
\end{table}
\end{minipage}

\begin{table}[H]
\begin{tabular}{ccccl}
\toprule
$i$ & $j$ & $k$ & $n$ & $\ell$ \\ 
\midrule
1 & 1 & 742 & 0 & 1\\
0 & 0 & 0 & 859 & 1\\
-2 & -2 & 0 & -2 & 0\\
-3 & -3 & 1 & -3 & 0\\
\bottomrule
\end{tabular}
\hfill
\begin{tabular}{ccccl}
\toprule
$i$ & $j$ & $k$ & $n$ & $\ell$ \\ 
\midrule
0 & 0 & 21 & 0 & 1\\
1 & 1 & 115 & 38 & 1\\ 
5 & 15 & 0 & 5 & 1\\
5 & 0 & 1 & 4 & 0\\
6 & 1 & 0 & 4 & 0\\
\bottomrule
\end{tabular}
\hfill
\begin{tabular}{ccccl}
\toprule
$i$ & $j$ & $k$ & $n$ & $\ell$ \\ 
\midrule
0 & 0 & 21 & 0 & 1\\
1 & 1 & 115 & 38 & 1\\
373 & 374 & -3 & 372 & 0\\
\bottomrule
\end{tabular}
\caption{Classification problems generated by \lig.}
\label{tab:ligdata}
\end{table}

Here, the label $\ell=1$ corresponds to positive examples and $\ell=0$ corresponds to negative examples. Our learner outputs $i\leq j$ as the classifier for the binary classification problem in Table~\ref{tab:initdata}. We run \lig with $\vec{r}=\{i,j\}$ as relevant variables and the rest of the variables marked irrelevant. Note that this set of relevant variables is incorrect and this instance of \lig will fail. In parallel to \lig, we continue improving our set of relevant variables.

The predicate $i\leq j$ {\em separates} the positives from the negatives: it includes, \ie, is {\it true} for all positive examples and excludes, \ie, is {\it false} for all negative examples in the data. Ideally, we want the classifier to {\em generalize} well: it should not happen that if we generate few more examples then the classifier can no longer separate the positives from the negatives.
Next, we check for the {\em robustness} of this separator by checking for existence of positive states that it excludes or negative states that it includes.  The former are generated via satisfying assignments of $\mathit{Pre}(\vec{x})\wedge\mathit{Trans}(\vec{x},\vec{x}')\wedge i'>j'$ and the latter from $\neg\mathit{Post}(\vec{x}')\wedge\mathit{Trans}(\vec{x},\vec{x}')\wedge i\leq j$ and are shown in Table~\ref{tab:ssdata}. Note that no new positive examples are added in this step as the former predicate is unsatisfiable. 

Next, we use the learner to find a classifier using the data in Table~\ref{tab:initdata} and Table~\ref{tab:ssdata}. We repeat these steps till an instance of \lig succeeds. Here, these iterations end at $i\leq j+k\wedge i\leq n+1$ which labels $i,j,k,n$ as relevant and $y$ as irrelevant. Note that any syntactic slicing-based technique would mark $y$ as relevant but the semantic data guides our learner to determine the irrelevance of $y$. Next, we show how \lig (with our improvements) successfully infers an invariant $I$ with this set of relevant variables. 
 
 \lig breaks down the process of finding $I$ into two steps. First, it creates many small binary classification problems. For each such problem, a {\em feature synthesizer} generates a {\em feature} that separates the positives from the negatives. Second, the features are combined together using a Boolean function learner to generate a candidate invariant. \lig repeats these steps till a predicate that satisfies all the VCs is discovered. For our example, \lig generates the classification problems in Table~\ref{tab:ligdata} (See~\cite{pie} for how \lig generates these problems). 
 Our contribution lies in using our learner to find features for each of these problems rather than using \lig's exhaustive enumeration based feature synthesizer.
 Here, our learner generates the following features for these three problems: $i\geq 0$, $i\leq j$, and $k\geq 0$. 
The Boolean function learner combines these features to generate the following candidate invariant $i\geq 0\wedge i\leq j\wedge k\geq 0$
that satisfies all the VCs.
This inductive invariant shows that the assertion in Figure~\ref{fig:motivating-example} holds for all possible inputs.

%% file: figures/motivating-example.tex
\begin{figure}%
    \vspace{-1.25em}\centering%
    \begin{algorithmic}[1]\linespread{1.05}\selectfont
        \Assume $( k\geq 0\wedge n\geq 0 )$
        \State $i = j  = 0$
        \While {$( i \leq n )$}
            \State $(i, j, y) \gets (i + 1, j + 1, i\times j)  $
        \EndWhile
        \Assert $( i+j+k\geq 2n\vee y \geq n^2)$
    \end{algorithmic}
    \caption{The C-program of working example.}
    \label{fig:motivating-example}
\end{figure}

%% file: sections/background.tex
\section{Background}%
\label{sec:background}

In this section, we formally define the problem of verifying correctness of programs using loop invariants, describe how invariant inference can be considered as a binary classification problem, and then describe how \lig reduces this classification problem to many small binary classification problems.

\subsection{Program Verification and loop invariants}%
\label{sec:background.verification}

The first step in program verification is defining a \emph{specification} for the desired property.
Typically~\citep{web/sv-comp,web/sygus-comp} this is provided as a pair of logical formulas ---
\begin{andlist}
    \item a \emph{precondition} that constrains the initial state of the program
    \item a \emph{postcondition} that validates the final state after execution of the program
\end{andlist}.
Many programming languages support the \pseudo{assume} and \pseudo{assert} keywords,
where $\pseudo{assume}(\phi)$ silently halts executions that satisfy $\neg\phi$ and executing $\pseudo{assert}(\phi)$ with a state that satisfies $\neg\phi$ raises an exception.
For example, \cref{fig:motivating-example} shows a program having a loop where the initial values are specified by initializations/\pseudo{assume} statements and the postcondition is specified using the \pseudo{assert} in the last line. 
Given such a specification, we define the verification problem as:

\begin{definition}[Program Verification]\label{def:verification}%
    Given a program \program{} and a specification consisting of a pair of formulas ---
    a precondition \pre{} and a postcondition \post{},
    the verification problem is to prove that
    for all executions starting from states that satisfy \pre, the states obtained after executing \program{} satisfy \post.
\end{definition}

In Floyd-Hoare logic (FHL)~\citep{floyd,hoare},
this problem is abbreviated to the formula $\triple{\pre}{\program}{\post}$, called a \emph{Hoare triple}.
We say that a Hoare triple is \emph{valid} if the correctness of \program{} can be provably demonstrated.
For example, while $\triple{x < 0}{\texttt{y $\gets$ -\,x}}{y > 0}$ is valid,
$\triple{x < 0}{\texttt{y $\gets$ x + 1}}{y < 0}$ is not. 
FHL offers initial theoretical underpinnings for automatic verification
by providing a set of inference rules that can be  used on the program structure.
Today, state-of-the-art verification tools have mechanized these rules and apply them automatically.

However, the FHL inference rules can automatically be applied only for validating Hoare triples that are defined on loop-free programs.
Applying these rules on a loop requires an additional parameter called a \emph{loop invariant} ---
a predicate over the program state that is preserved across each iteration of the loop.
To establish the validity of a Hoare triple,
the FHL require a loop invariant to satisfy three specific properties,
and a predicate that satisfies all three is called a \emph{sufficient loop invariant}.

\begin{definition}[Sufficient Loop Invariant]\label{def:invariant}%
    Consider a simple loop, $\pseudo{while}\ \,G\ \,\pseudo{do}\ \,S$,
    which executes the statement $S$ until the condition (loop guard) $G$ holds and then it halts.
    Then, for the Hoare triple $\triple{\pre}{\pseudo{while}\ \,G\ \,\pseudo{do}\ \,S}{\post}$ to be valid,
    there must exist a predicate \inv{} that satisfies:
    \begin{enumerate}[leftmargin=28mm, itemsep=0.3125em, topsep=0.25em]
        \item[$\vc{pre}$\,:\hspace{1em}] $\pre \implies \inv$, \ie,\ \inv\ must hold immediately before the loop
        \item[$\vc{ind}$\,:\hspace{1em}] $\triple{G \wedge \inv}{S}{\inv}$, \ie\ \inv\ must be inductive (hold after each iteration)
        \item[$\vc{post}$\,:\hspace{1em}] $\,  \inv \implies G\vee \post$, \ie,\ \inv\ must certify the postcondition upon exiting the loop
    \end{enumerate}
    These three properties are called the \emph{verification conditions} (VCs) for the loop.
    Any predicate \,\inv{} that satisfies the first two VCs is called a \emph{loop invariant}.
    A loop invariant that also satisfies the third VCs said to be \emph{sufficient} (for proving the correctness of the Hoare triple).
    In this paper, we use {\em invariants} to denote sufficient loop invariants for brevity.
\end{definition}

Thanks to efficient theorem provers~\citep{z3,cvc4},
today it is possible to automatically check if a given predicate is indeed an invariant.
However, automatically finding an invariant for arbitrary loops is undecidable in general,
and even small loops are challenging for state-of-the-art tools.
The invariant inference track of the syntax guided synthesis competition has hundreds of benchmarks where each benchmark provides a $\vc{pre}$, a $\vc{ind}$, and a $\vc{post}$
as logical formulas. Different tools compete to {\em solve} these problems, \ie, to infer the invariants every year.

\subsection{Data-Drive Invariant Inference}%
\label{sec:background.classification}

An invariant can be viewed as a zero-error classifier ---
it should demonstrate that the set of possible reachable states at the entry to a loop (called \emph{loop-head states}) are disjoint
from the bad states that violate the postcondition;
thus establishing that the postcondition is satisfied for all executions.

\input{figures/classification}

Consider verifying our motivating example from \cref{fig:motivating-example}.
We visualize the classification problem in \cref{fig:classification}.
\vc{pre} and \vc{ind} from \cref{def:invariant} require \inv{} (\textcolor{blue!64!black}{dashed blue ellipse})
to capture all possible loop-head states (\textcolor{cyan!32!black}{cyan dots}).
These include states satisfying the precondition \pre{} (\textcolor{green!40!black}{green circle}),
\eg, $(i = j = 0, n = 2)$ appearing before the first iteration,
and the subsequent states after each iteration (indicated by the arrows),
\eg, $(i = j = 1, n = 2, y = 0)$, $(i = j = 2, n = 2, y = 1)$ etc.
The $\neg G \wedge \neg \post$ space (\textcolor{red!64!black}{red rectangle}) denotes the states
violating the postcondition, \eg, $(i = j = 2, k = 0, n = 1, y = -1)$.
\vc{post} forces \inv{} to be disjoint with this space.
An invariant \inv{} that satisfies the VCs guarantees that no execution starting from $\pre$ would terminate at a state that violates the desired postcondition \post{}.

To infer invariants, we can label examples of loop-head states as positive and satisfying assignments of $\neg G\wedge \neg\post$ as negative and use a classification algorithm to separate these. The output classifier is a candidate invariant. If the candidate satisfies all the VCs then we have succeeded in inferring an invariant. If some VC is violated then SMT solvers can produce counterexamples which can be added to positive or negative examples to generate another candidate. 
Since the actual invariant can be complex, prior work has explored increasingly complex learning algorithms including support vector machines~\cite{iac,selective}, decision trees~\cite{icedt,ddchc}, algorithms for learning Boolean combinations of half-spaces~\cite{ssd}, Metropolis-Hastings sampling~\cite{c2i}, Gibbs sampling~\cite{popl07}, SMT-based constraint solving~\cite{ice}, and, finally, neural networks~\cite{code2inv,cln2inv}. 
An alternative approach was proposed by~\cite{pie} where this classification problem is decomposed into smaller more tractable classification problems that can be solved by simple learning algorithms. This approach is implemented in the tool \lig that \oasis builds upon.

\subsection{\lig}%
\label{sec:background.cegis}

\lig~\citep{pie} is
a state-of-the-art data-driven invariant inference tool. It consists of a learner and a teacher that interact with each other.
The teacher has access to an SMT solver and can verify loop-free programs. In particular, given a candidate invariant \inv{} generated by the learner, it can check the VCs and if some VC fails then it returns a program state as a counterexample.
\lig uses a multi-stage learning technique that composes the candidate invariant out of several predicates, known as features, learned over smaller subproblems.
\cref{algo:loopinvgen} outlines this framework.

\input{figures/loopinvgen}

The main \textsc{Infer} procedure is invoked with a Hoare triple $\mathcal{L} \equiv \triple{\pre}{\pseudo{while}\ G\ \pseudo{do}\ S}{\post}$,
and a set $\mathcal{P}$ of reachable program states. 
Here, we assume the loop-body $S$ to be loop-free\footnote{\lig does handle multiple and nested loops~\cite{pie}.}.
The program states are sampled at random by running the loop for a few iterations~\cite{qcheck}. All the states that \lig deals with are total maps that map all variables to some integers.
The {\sc Check}($B$) procedure is a call to the teacher that invokes an SMT solver to check if $B$ is valid. If $B$ is valid the call returns $\bot$ otherwise it returns a (complete) satisfying assignment of $\neg B$.

Line 2 performs a sanity check: if $\pre\wedge\neg G\wedge\neg \phi$ is satisfiable then the input Hoare triple is invalid and no invariant exists. 
\lig starts with a weak candidate invariant $\inv{} \equiv (\neg\, G \implies \post)$,
and iteratively strengthens it (line 17) for inductiveness.
These choices ensures that all candidate invariants \inv{} satisfy $\vc{post}$.
Lines 5 and 15 additionally check for \vc{pre} and \vc{ind} respectively,
and add appropriate counterexamples.
While a violation of \vc{pre} adds a \emph{positive} example,
a violation of \vc{ind} adds a \emph{negative} example.
Since the loop body $S$ is loop free, \vc{ind} can be encoded as an SMT formula (through a weakest precondition computation) whose validity ensures the validity of \vc{ind}.
The  lines 10 -- 14 indicates the key learning subcomponents.

In line 11, the \textsc{Conflict} procedure selects two sets $P\subseteq \mathcal{P}$ and $N \subseteq \mathcal{N}$
that are \emph{conflicting}, \ie, these positive and negative examples are indistinguishable modulo $\mathcal{F}$, the set of current features.
That is for all features $f\in\mathcal{F}.\forall x,y\in P\cup N. f(x)=f(y)$. For such $P$ and $N$, line 13 learns a  feature that separates $P$ and $N$ by invoking the learner \textsc{Learn}. In \lig, the learner is implemented using Escher~\cite{escher} that exhaustively enumerates all predicates over all variables in increasing size till it finds a feature $f$ that separates $P$ and $N$, and this $f$ is added to $\mathcal{F}$. The loop in lines 10--13 has the following postcondition: $\forall x\in \mathcal{P}.\forall y\in \mathcal{N}.\exists f\in \mathcal{F}.f(x)\ne f(y)$, \ie, for every positive example $x$ and every negative example $y$, there is a feature $f$ that separates $x$ and $y$.
Once $\mathcal{F}$ has enough features, line 14  uses a standard Boolean-function learner~\citep{mitchell,pie} \textsc{BoolCombine} to learn $\delta$, a Boolean combination of these features, that separates $\mathcal{P}$ and $\mathcal{N}$. Then \lig logically strengthens  the candidate invariant $\inv{}$ by conjoining it with $\delta$. 
For more details on this framework, we refer to the \lig paper~\citep{pie}. In particular,~\cite{pie} shows that breaking the binary classification problem of separating $\mathcal{P}$ and $\mathcal{N}$ by a candidate invariant into the two step approach of first inferring features that separate $P\subseteq \mathcal{P}$ and $N \subseteq \mathcal{N}$ and then combining the features is an effective approach to invariant inference. The features are usually much simpler than the invariants, which makes inferring features much more tractable than inferring candidate invariants.

Next, we discuss our contributions: the inference of relevant variables and the changes \oasis makes to \lig followed by our ILP-based learning (\Cref{sec:learning}).

%% file: figures/classification.tex
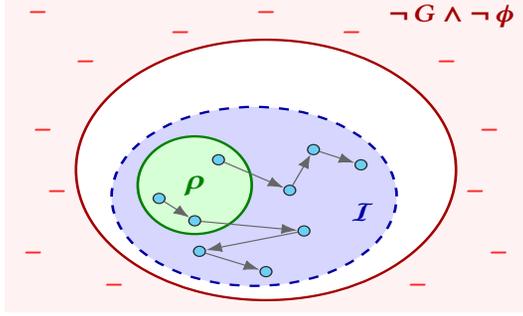
\begin{figure}%
    \vspace{-0.3125em}\centering%
    \resizebox{0.5\linewidth}{0.3\linewidth}{%
        \begin{tikzpicture}[
            dot/.style={circle,inner sep=1.25pt,fill=cyan!50!white,draw=gray!32!black,name=#1},
            dash/.style={circle,inner sep=0pt,name=#1,label={\textcolor{red!85!white}{\textbf{--}}}},
            arrow/.style={-Latex, draw=gray!80!black}
        ]
            \draw[fill=red!5!white, draw=none]
                (-2.75,-1.75) rectangle (2.75,2.125)
                node[anchor=south east, yshift=-5mm, xshift=-0.125mm]
                {\textcolor{red!64!black}{\small $\bm{\neg\, G \wedge \neg\, \post}$}};

            \draw[thick, fill=white, draw=red!64!black]
                (0,0) ellipse (2cm and 16mm);
            \draw[thick, dashed, fill=blue!16!white, draw=blue!64!black]
                (-0.125,-0.325) ellipse (15mm and 11mm)
                node[anchor=west, xshift=9mm, yshift=-2mm]
                {\textcolor{blue!64!black}{\small $\bm{\inv}$}};
            \draw[thick, fill=green!16!white, draw=green!50!black]
                (-0.75,-0.1875) circle (6mm)
                node[anchor=center]
                {\textcolor{green!50!black}{$\bm{\pre}$}};
            
            \node [dot=A1] at (-0.5,0.125) {};
            \node [dot=A2] at (0.25,-0.25) {};
            \node [dot=A3] at (0.5,0.25) {};
            \node [dot=A4] at (1,0.0625) {};

            \draw [arrow] (A1) -- (A2);
            \draw [arrow] (A2) -- (A3);
            \draw [arrow] (A3) -- (A4);

            \node [dot=B1] at (-1.125,-0.35) {};
            \node [dot=B2] at (-0.75,-0.625) {};
            \node [dot=B3] at (0.4,-0.75) {};
            \node [dot=B4] at (-0.7,-1) {};
            \node [dot=B5] at (0,-1.25) {};

            \draw [arrow] (B1) -- (B2);
            \draw [arrow] (B2) -- (B3);
            \draw [arrow] (B3) -- (B4);
            \draw [arrow] (B4) -- (B5);

            \node [dash=N1] at (-2.35,0.25) {};
            \node [dash=N2] at (-2.45,-1.25) {};
            \node [dash=N3] at (-1.6,-1.65) {};
            \node [dash=N4] at (-0.9,1.45) {};
            \node [dash=N6] at (-2.4,1.7) {};
            \node [dash=N7] at (-1.9,1.1) {};
            \node [dash=N8] at (-2.2,-0.5) {};

            \node [dash=N1x] at (2.35,0.25) {};
            \node [dash=N2x] at (2.45,-1.25) {};
            \node [dash=N3x] at (1.6,-1.65) {};
            \node [dash=N4x] at (0.9,1.45) {};
            \node [dash=N5x] at (0,1.7) {};
            \node [dash=N7x] at (1.9,1.1) {};
            \node [dash=N8x] at (2.2,-0.5) {};
        \end{tikzpicture}}
    \caption{A sufficient loop invariant can be viewed as a classifier for states.}%
    \label{fig:classification}
\end{figure}

%% file: figures/loopinvgen.tex
\begin{algorithm}[t]
    \caption{The \lig algorithm~\citep{pie}. The teacher is {\sc Check} and the learner is {\sc Learn}.}\label{algo:loopinvgen}
    \begin{algorithmic}[1]
        \small\linespread{1}\selectfont
        \FunctionX{Infer}{$\triple{\pre}{\pseudo{while}\ G\ \pseudo{do}\ S}{\post}$, $\mathcal{P}$}
            \IfThen{{\sc Check}$(\pre \implies (G \vee \post))\ne \bot$}{\Return $\false$}
            \State $\inv \gets (\neg\, G \implies \post)$
            \While {\true}
                \State $c \gets \textnormal{\sc Check}(\pre \implies \inv)$
                \IfThen {$c \neq \bot$} {\Return $\textnormal{\sc Infer}(\triple{\pre}{\pseudo{while}\ G\ \pseudo{do}\ S}{\post}, \mathcal{P} \,\cup\, \{c\})$}
                \State $\mathcal{N} \gets \{\}$
                \While {\true}
                    \State $\mathcal{F} \gets \{\}$
                    \While {\true}
                        \State $(P,N) \gets \textnormal{\sc Conflict}(\mathcal{P}, \mathcal{N}, \mathcal{F})$
                        \If {$P = N = \{\}$} \textbf{break}
                        \Else\ $\mathcal{F} \gets \mathcal{F} \,\cup\, \textcolor{blue!75!black}{\textnormal{\sc Learn}}(P, N)$
                        \EndIf
                    \EndWhile
                    \State $\delta \gets \textnormal{\sc BoolCombine}(\mathcal{F})$
                    \State $c \gets \textnormal{\sc Check}( \triple{\delta\wedge G \wedge \inv}{S}{\inv})$
                    \IfThen {$c \neq \bot$} {$\mathcal{N} \gets \mathcal{N} \,\cup\, \{c\}$}
                \EndWhile
                \State $\inv \gets (\inv \wedge \delta)$
                \IfThen {$\delta = \true$} {\Return \inv}
            \EndWhile
        \EndFunction
    \end{algorithmic}
\end{algorithm}

%% file: sections/overview.tex
\section{\oasis\ Framework}%
\label{sec:overview}

In this section, we overview our approach for accelerating invariant inference using a set of \emph{relevant} variables.
First, we define the state space for programs
and describe our encoding of the verification conditions described in \cref{def:invariant}.
We then describe the notion of relevant variables for a program verification problem,
and present our approach for inferring sufficient loop invariants using these relevant variables.

\subsection{Notation}%
\label{sec:overview.notation}

Given a program $\program$ we write $\Vars[\program]$,
to denote the sequence $\Tuple{x_1}{\ldots}{x_n}$ of variables appearing in it.
We omit the subscript $\program$ and simply write $\Vars$ when the program is clear from context.
A program state for $\program$, denoted $\PState = \Tuple{v_1}{\ldots}{v_n}$,
is a sequence of values assigned to the program variables ---
any subset of these values may be \emph{irrelevant} (denoted $\top$).
A program state $\PState$ is said to be \emph{total} if it does not contain $\top$,
and is said to be \emph{partial} otherwise.
Finally, we use the shorthand $\Assign{\Vars}{\PState}$
to denote the value assignment predicate $(\Conj{x_1 = v_1}{\cdots}{x_n = v_n})$,
where irrelevant values ($\top$) are simply dropped,
\eg, $(\Tuple{x_1}{x_2}{x_3} \mapsto \Tuple{v_1}{\top}{v_3}) \equiv (\Conj{x_1 = v_1}{x_3 = v_3})$.

Although, the techniques described in this work can be easily extended to programs containing multiple and nested loop,
for simplicity, we consider verifying our single-loop program from the previous section:
$\mathcal{L} \equiv \triple{\pre}{\pseudo{while}\ G\ \pseudo{do}\ S}{\post}$.
We formally model the loop in our program $\program$ as a transition relation $\TransRel[\program]$ over program states.
Two states $\PState_1$ and $\PState_2$ are related by $\TransRel[\program]$ iff
a single iteration of the loop body ($S$) transitions the state $\PState_1$ to $\PState_2$. We need $\TransRel[\program]$ to be a relation as programs can have non-determinism.
Formally,
\[
    \TransRel[\program](\PState_1, \PState_2) \iff \triple{G \wedge \Assign{\Vars}{\PState_1}}{S}{\Assign{\Vars}{\PState_2}}
\]
Similarly, we model the precondition and postcondition as unary predicates on program states:
\[
    \PreRel[\program](\PState) \iff \Assign{\Vars}{\PState} \wedge \pre
    \qquad\qquad\qquad
    \PostRel[\program](\PState) \iff \Assign{\Vars}{\PState} \wedge \neg G \implies \post
\]
We omit the subscript $\program$ and simply write $\PreRel$, $\TransRel$
and $\PostRel$ when the program $\program$ is clear from context.
These relations together define a program verification (or equivalently, a sufficient loop invariant inference) problem.
Indeed, this encoding of program verification problems is a commonly used language-agnostic intermediate representation.
Moreover, existing program analysis tools can automatically generate these relations
from high-level programs and their formal specifications.
This representation facilitates the use of off-the-shelf SMT solvers.

A sufficient loop invariant that establishes the correctness of the Hoare triple $\mathcal{L}$
is also a predicate defined over states of the program $\program$.
Such an invariant is required to satisfy the following verification conditions from \cref{def:verification}
in terms of the $\PreRel$, $\TransRel$\ and $\PostRel$ relations above:\\[-0.5em]

\begin{minipage}{0.23\textwidth}%
    \centering%
    $\forall \PState \ldotp \PreRel(\PState) \implies \InvRel(\PState)$ \\
    {\small(\vc{pre})}
\end{minipage}\hfill%
\begin{minipage}{0.4\textwidth}%
    \centering%
    $\forall \PState, \PState' \ldotp \InvRel(\PState) \wedge \TransRel(\PState, \PState') \implies \InvRel(\PState')$ \\
    {\small(\vc{ind})}
\end{minipage}\hfill%
\begin{minipage}{0.24\textwidth}%
    \centering%
    $\forall \PState \ldotp \InvRel(\PState) \implies \PostRel(\PState)$ \\
    {\small(\vc{post})}
\end{minipage}

\begin{example}
    Consider again our motivating example from \cref{fig:motivating-example} where $\Vars = \Tuple{i}{j}{k}{n}{y}$.
    We use $\PState$ and $\PState'$ to denote the tuples $\Tuple{v_i}{v_j}{v_k}{v_n}{v_y}$
    and $\Tuple{v'_i}{v'_j}{v'_k}{v'_n}{v'_y}$ of values respectively.
    The following \PreRel, \TransRel\ and \PostRel\ relations encode the verification problem:
    \[\def\arraycolsep{3pt}%
    \begin{array}{rcl}
        \PreRel(\PState) & \triangleq
                         & \Conj{(v_i = v_j = 0)}
                                {(v_k \geq 0)}
                                {(v_n \geq 0)} \\[0.5em]
        \TransRel(\PState, \PState') & \triangleq
                                     & \Conj{(v'_k = v_k)}
                                            {(v'_n = v_n)}
                                            {\left(
                                                (v_i \leq v_n)
                                                \implies
                                                \big(
                                                    \Conj{v'_i = v_i + 1}
                                                         {v'_j = v_j + 1}
                                                         {v'_y = v_i \cdot v_j}
                                                \big)
                                            \right)}\\[0.5em]
        \PostRel(\PState) & \triangleq
                          & \neg (v_i \leq v_n)
                            \implies
                            \left(
                                \Disj{(v_i + v_j + v_k \geq 2 \cdot v_n)}
                                     {(v_y \geq v_n^2)}
                            \right)
    \end{array}\]
\end{example}

\subsection{Relevant Variables}%
\label{sec:overview.using}

Scalability is a major challenge for existing data-driven invariant inference techniques.
As number of variables increases the performance of these techniques degrades rapidly,
although in many cases a sufficient invariant for verifying these programs contains only a small number of variables.
We propose a novel technique that first identifies a small subset of variables
over which a sufficient loop invariant is likely to exist.
Then, it simultaneously refines this subset and searches for a sufficient invariant till one is found.

\input{figures/oasis}

Our core framework, called \oasis, is outlined in \cref{algo:oasis}.
\oasis accepts the standard set of arguments for a data-driven verification technique (discussed in \Cref{sec:background.cegis}) ---
a verification problem (encoded as a triple $\Tuple{\PreRel}{\TransRel}{\PostRel}$),
and some sampled positive ($\PosPStates$) and negative ($\NegPStates$) program states typically sampled randomly.
We first invoke the \AlgoStyle{\LearnClassifier}\ function with these sampled states
to learn a predicate $\Classifier$ that separates $\PosPStates$ and $\NegPStates$, \ie,
\[
    \left(\forall \PState \in \PosPStates \ldotp \Classifier(\PState) \right)
    \wedge
    \left(\forall \PState \in \NegPStates \ldotp \neg \Classifier(\PState) \right)
\]
We detail the \AlgoStyle{\LearnClassifier}\ function in \Cref{sec:learning},
which utilizes machine-learning techniques to efficiently find a sparse separator for $\PosPStates$ and $\NegPStates$.
In line 3 we drop irrelevant variables, those that do not affect the prediction of the classifier over  $\PosPStates \cup \NegPStates$,
and consider the remaining variables $\vec{r} \subseteq \Vars$ to be a candidate set of relevant variables.
In \Cref{sec:example} we show some examples of classification problems, the learned classifiers and relevant variables.

After a set $\vec{r}$ of relevant variables is identified, in lines 3\,--\,12,
we try to refine the set of relevant variables and find a sufficient invariant over them in parallel.
In particular, we execute the following three threads in parallel:
\begin{enumerate}
    \item one that attempts to find a \emph{positive} state misclassified by the classifier
    \item one that attempts to find a \emph{negative} state misclassified by the classifier
    \item one that runs invariant inference using the currently identified relevant variables
\end{enumerate}
\citet{lessismore} showed that classifiers can be refined by sampling near the classification boundary. This idea is used in
\citet{selective}, which showed that when compared to random sampling,
active learning improves the quality of sampled program states and accelerates the search for sufficient invariants. 
Threads 1 and 2 are responsible for an active-learning-based refinement of the relevant variables set,
and thread 3 attempts to find a sufficient invariant over these variables, if there exists one.
Next, we detail our active learning strategy (\Cref{sec:overview.refinement})
and our relevance-aware invariant inference algorithm \AlgoStyle{\RelInfer} (\Cref{sec:overview.inference}).
Note that the \AlgoStyle{\RelInfer}\ thread is run with a timeout of $\tau$
so that long-running inference threads are automatically cleaned up
as we spin up more threads with refined sets of relevant variables.

\subsection{Refining Relevant Variables}%
\label{sec:overview.refinement}

\input{figures/refinement}

We now detail our procedures for refining a set of relevant variables.
The \AlgoStyle{\FindPosCounterExample}\ and \AlgoStyle{\FindNegCounterExample}\ procedures,
which run in threads 1 and 2 respectively, are outlined in \cref{algo:refinement}.
Each of these procedures returns a program state that is misclassified by the current classifier $\Classifier$,
which is then used to learn a new classifier, and thus a new set of relevant variables.

The \AlgoStyle{\FindPosCounterExample}\ procedure identifies \emph{positive} misclassifications ---
a reachable program state $\PState$ that the classifier labels as a negative state, \ie, $\neg \Classifier(\PState)$.
To identify such states, we gradually expand the frontier of reachable states starting from the precondition \PreRel
and then repeatedly applying the transition relation \TransRel.
In line 2, we construct the predicate $\AlgoStyle{Reachable}(k)$ that captures all states
that are reachable in exactly $k$ applications of the transition relation, \ie, $k$ iterations of the loop.
In line 3, we check if all such states are subsumed by the current classifier.
Upon finding a counterexample, in line 4, we return the misclassified state.

The \AlgoStyle{\FindNegCounterExample}\ procedure works in a very similar manner and identifies \emph{negative} misclassifications ---
a bad program state $\PState$ (one that would lead to violation of the final assertion)
that the classifier labels as a positive state, \ie, $\Classifier(\PState)$.
To identify such states, we gradually expand the frontier of known bad states
starting from those that violate the postcondition \PostRel
and then repeatedly reversing the transition relation \TransRel.
In line 6, we construct the predicate $\AlgoStyle{Bad}(k)$ that captures all states
that lead to state to an assertion violation in exactly $k$ applications of the transition relation, \ie, $k$ iterations of the loop.
In line 7, we check if all such states are excluded by the current classifier.
Upon finding a counterexample, in line 8, we return the misclassified state.

Although these procedures can be computationally expensive, our implementation caches the results of intermediate queries for reuse. In particular, unsatisfiable paths are generated at most once.
\input{figures/relevant-infer}

\subsection{Invariant Inference with Relevant Variables}%
\label{sec:overview.inference}

Once we have a set of relevant variables from the learned classifier,
we run our invariant inference algorithm (in thread 3) with these variables together with
all the positive states ($\PosPStates$) and negative states ($\NegPStates$) sampled so far.
In \cref{algo:relevant-infer} we outline this algorithm.
They key difference with respect to \cref{algo:loopinvgen} is the use of $\vec{r}$ ---
the set of relevant variables.
While \cref{algo:loopinvgen} learns features over all variables $\vec{x}$ in the program,
\cref{algo:relevant-infer} only learns features over $\vec{r}$,
which is provided to the \AlgoStyle{\learn}\ procedure in line 11.
In the next section, we detail this relevance-aware learning procedure.

%% file: figures/oasis.tex
\begin{algorithm}[t]
    \caption{\oasis framework for scaling loop invariant inference}\label{algo:oasis}
    \begin{algorithmic}[1]
        \small\linespread{1}\selectfont
        \FunctionX{\oasis}{$\langle \PreRel,\TransRel,\PostRel \rangle$ : Verification Problem, $\PosPStates$ : States, $\NegPStates$ : States}
            \State Classifier $\Classifier \gets \Call{\LearnClassifier}{\PosPStates, \NegPStates}$
            \IfThen {$\Classifier = \bot$} \Return $\bot$
            \vspace*{0.25em}
            \State Variables $\vec{r} \gets \Call{FilterVariables}{\Classifier}$
            \ParDo
                \Thread{1}
                    \State $\PState \gets \Call{\FindPosCounterExample}{\langle \PreRel,\TransRel,\PostRel \rangle, \Classifier}$
                    \IfThen {$\PState \neq \bot$}
                        \Return $\Call{\oasis}{\langle \PreRel,\TransRel,\PostRel \rangle, \PosPStates \cup \{ \PState \}, \NegPStates}$
                \EndThread
                \Thread{2}
                    \State $\PState \gets \Call{\FindNegCounterExample}{\langle \PreRel,\TransRel,\PostRel \rangle, \Classifier}$
                    \IfThen {$\PState \neq \bot$}
                        \Return $\Call{\oasis}{\langle \PreRel,\TransRel,\PostRel \rangle, \PosPStates, \NegPStates \cup \{ \PState \}}$
                \EndThread
                \Thread{3}
                    \State $\InvRel \gets \Call{\RelInfer}{\langle \PreRel,\TransRel,\PostRel \rangle, \PosPStates, \NegPStates, \vec{r}} \, \big\vert_{\,\textsf{timeout = }\tau}$
                    \IfThen {$\InvRel \neq \bot$}
                        \Return $\InvRel$
                \EndThread
            \EndParDo
        \EndFunction
    \end{algorithmic}
\end{algorithm}

%% file: figures/refinement.tex
\begin{algorithm}[t]
    \caption{Procedures for refinement of candidate relevant variables}\label{algo:refinement}
    \begin{algorithmic}[1]
        \small\linespread{1}\selectfont
        \FunctionX{\FindPosCounterExample}{$\langle \PreRel,\TransRel,\PostRel \rangle : \text{Verification Problem}$, $\Classifier : \text{Predicate}$}
            \For {$k = 0$ \textbf{to} $\infty$}
                \State Predicate $\AlgoStyle{Reachable}(k) \triangleq \PreRel(\PState_0) \wedge \TransRel(\PState_0, \PState_1) \wedge \cdots \wedge \TransRel(\PState_{k-1},\PState_k)$
                \State Counterexample $\bm{c} \gets \AlgoStyle{Check}(\forall \PState_0, \ldots, \PState_k
                                                                      \ldotp
                                                                      \AlgoStyle{Reachable}(k)
                                                                      \implies
                                                                      \Classifier(\PState_k))$
                \IfThen {$\bm{c} \neq \bot$}
                    \Return $\bm{c}[\PState_k]$
            \EndFor
        \EndFunction
        \vspace*{1em}
        \FunctionX{\FindNegCounterExample}{$\langle \PreRel,\TransRel,\PostRel \rangle : \text{Verification Problem}$, $\Classifier : \text{Predicate}$}
            \For {$k = 0$ \textbf{to} $\infty$}
                \State Predicate $\AlgoStyle{Bad}(k) \triangleq \neg \PostRel(\PState_k) \wedge \TransRel(\PState_{k-1}, \PState_k) \wedge \cdots \wedge \TransRel(\PState_0,\PState_1)$
                \State Counterexample $\bm{c} \gets \AlgoStyle{Check}(\forall \PState_0, \ldots, \PState_k
                                                                      \ldotp
                                                                      \AlgoStyle{Bad}(k)
                                                                      \implies
                                                                      \neg \Classifier(\PState_0))$
                \IfThen {$\bm{c} \neq \bot$}
                    \Return $\bm{c}[\PState_0]$
            \EndFor
        \EndFunction
    \end{algorithmic}
\end{algorithm}

%% file: figures/relevant-infer.tex
\begin{algorithm}
    \caption{A loop invariant inference algorithm that utilizes relevant variable information}\label{algo:relevant-infer}
    \begin{algorithmic}[1]
        \small\linespread{1}\selectfont
        \FunctionX{\RelInfer}{$\langle \PreRel,\TransRel,\PostRel \rangle$ : Verification Problem, $\PosPStates$ : States, $\NegPStates$ : States, $\vec{r}$ : Variables}
            \IfThen{$\AlgoStyle{Check}(\forall \PState \ldotp \PreRel(\PState) \implies \PostRel(\PState)) \neq \bot$}
                   {\Throw{``No Solution!''}}
            \State Predicate $\inv \gets \PostRel$
            \While {\true}
                \IfThen {$\AlgoStyle{Check}(\forall \PState, \PState' \ldotp \InvRel(\PState) \wedge \TransRel(\PState, \PState') \implies \InvRel(\PState')) = \bot$} {\Return $\InvRel$}
                \State States $\mathcal{P}, \mathcal{N} \gets \PosPStates, \NegPStates$
                \While {\true}
                    \State Features $\mathcal{F} \gets \{\}$
                    \While {\true}
                        \State States $(P,N) \gets \textnormal{\sc Conflict}(\mathcal{P}, \mathcal{N}, \mathcal{F})$
                        \If {$P = N = \{\}$} \textbf{break}
                        \Else\ $\mathcal{F} \gets \mathcal{F} \cup \Call{Learn}{P, N, \vec{r}}$
                        \EndIf
                    \EndWhile
                    \State Predicate $\delta \gets \Call{BoolCombine}{\mathcal{F}}$
                    \State Counterexample $\bm{c} \gets \AlgoStyle{Check}(\forall \PState, \PState' \ldotp \delta(\PState) \wedge \InvRel(\PState) \wedge \TransRel(\PState, \PState') \implies \InvRel(\PState'))$
                    \IfThen {$\bm{c} = \bot$} {\Break}
                    \State $\mathcal{N} \gets \mathcal{N} \,\cup\, \{c[\PState]\}$
                \EndWhile
                \State $\inv \gets (\inv \wedge \delta)$
                \State Counterexample $\PState \gets {\forall \PState \ldotp \PreRel(\PState) \implies \InvRel(\PState)}$
                \If {$\PState \neq \bot$}
                    \State $\InvRel \gets \PostRel$
                    \State $\PosPStates \gets \PosPStates \cup \{\PState\}$
                \EndIf
            \EndWhile
        \EndFunction
    \end{algorithmic}
\end{algorithm}

%% file: sections/learning.tex
\section{Classifier Learning}%
\label{sec:learning}

In this section, we formulate the problem of generating a classifier that separates positive program states from negative program states. By default, the output classifier predicate can use any of the program variables. If we restrict the classifier to use only a subset $\vec{r}$ of variables (\eg, the call to \AlgoStyle{\learn}\ procedure in line 11 of \cref{algo:relevant-infer}) then we first project the examples to $\vec{r}$ and then learn a classifier over the projected states.
Let $\x$  denote a vector of program variables that can occur in the classifier. In this section, we use the standard notation that bold letters denote vectors (\eg, $\mathbf{0}$ is a vector of all zeros).
We model the problem of inferring a classifier $h: \integer^{|\x|} \to \{\true{}, \false{}\}$
as a search problem over the following class of CNF predicates
with $\nc$ denoting the number of conjuncts and $\nd$ the number of disjuncts in each conjunct:
\begin{equation}
\intloopinv{} = \bigg\{  \underset{\ci \in [\nc]}{\bigwedge}\;\underset{\di \in [\nd]}{\bigvee}{\langle \w_{\ci\di}, \x\rangle + b_{\ci\di} > 0} \bigg\}.
\end{equation}
where $b\in\integer$ and $\langle \w, \x\rangle + b$ is an inner product between a vector $\w \in \integer^{|\x|}$ and $\x$. We use $[n]$ to denote the list $\{0,1,\ldots,n-1\}$.

Given a set of program states with corresponding labels,
our task is to find a classifier $h\in \intloopinv$ such that
\begin{andlist}
  \item it separates the positive states from the negative states
  \item it \emph{generalizes} to unseen program states
\end{andlist}
The first part is a \emph{search} question, whereas the second part suggests learning to choose simple and \emph{natural} predicates.
Note that the class of invariants $\intloopinv{}$ is very powerful -- one can trivially fit any given set of examples.
We make the following observations.
The search problem becomes meaningful on a given set of program states, if we restrict the predicate sizes (\ie, $\nc$ and $\nd$) to be small.
Furthermore, the coefficients are  often bounded by the constants occurring in the program.
Finally, and most importantly, we are not dealing with arbitrary predicate formulas,
but ones that have a nice conjunction-of-disjunctions structure.
These observations enable reformulating the search problem as an integer-linear programming (ILP) problem that can be \emph{efficiently} solved in practice for our benchmarks by off-the-shelf ILP solvers. 

Consider the search problem \inliststyle{(1)} above:
formally, we want to find a predicate $h \in \intloopinv{}$ that accurately classifies a given set of labeled program states $\{\pstate_\ii, y_\ii\}_{\ii=1}^\nii$, where $y_\ii \in \{0,1\}$.
It is convenient to think of $h$ as a tree of depth 3:
the program variables form the input layer to the linear inequalities,
which are grouped by $\bigvee$ operators to yield disjunctive predicates.
The root node is the $\bigwedge$ operator that represents conjunction of the predicates represented by the second layer.
The reduction of the search problem to ILP is given as follows.

(\textbf{Input layer: linear inequalities})
Write $\z_{\ii\ci\di} = \ind{\langle \w_{\ci\di}, \pstate_\ii\rangle  + b_{\ci\di} > 0}\in \{0,1\}$, where the indicator function $\ind{p}$ of a predicate $p$ maps $\true$ to 1 and $\false$ to 0.
This is captured by the following constraints, for a sufficiently large integer $M$:\\
\begin{align}\label{eqn:polyconst}
  \forall \ii \in [\nii], \ci \in [\nc], \di \in [\nd], \ \ \ \ -M (1-\z_{\ii\ci\di})\ \  < \ \ \langle \w_{\ci\di}, \pstate_\ii \rangle + b_{\ci\di} &\enskip\leq\enskip M \z_{\ii\ci\di}, \nonumber\\[0.125em]
  \w_{\ci\di} \in \integer^{|\x|}\ , b_{\ci\di} \in \integer\ , \z_{\ii\ci\di} &\enskip\in\enskip \{0,1\}.
\end{align}

(\textbf{Middle layer: Disjunctions})
Note that the value of the $\ci$-th conjunct on a given input $\pstate_\ii$ corresponds to summing $\z_{\ii\ci\di}=\ind{p_{\ii\ci\di}}$ over $\di$,
\ie, write $\ylor_{\ii\ci} = \ind{\bigvee_{\di \in [\nd]}{} p_{\ii\ci\di}}$.
This is captured by the constraint:\\
\begin{align}\label{eqn:disconst}
  \forall \ii \in [\nii], \ci \in [\nc],\ \ \ \  -M (1-\ylor_{\ii\ci})\ \  <\ \  \sum_{\di \in [\nd]} \z_{\ii\ci\di} &\enskip\leq\enskip M\ylor_{\ii\ci}, \nonumber\\[-0.25em]
  \ylor_{\ii\ci} &\enskip\in\enskip \{0,1\}.
\end{align}

(\textbf{Final layer: Conjunction})
The predicted label on a given input state is given by a conjunction of the above disjunctions.
Requiring that the predicted label match the observed label for each example is equivalent to the following constraints:\\
\begin{align}\label{eqn:conconst}
  \text{for } \ii \in [\nii] \text{ s.t. } y_\ii = 1,  \sum_{\ci \in [\nc]} \ylor_{\ii\ci} &\enskip\geq\enskip \nc\ , \nonumber\\
  \text{for } \ii \in [\nii] \text{ s.t. } y_\ii = 0,  \sum_{\ci \in [\nc]} \ylor_{\ii\ci} &\enskip\leq\enskip \nc - 1 \ .
\end{align}

The search problem can now be stated as the ILP problem:
\emph{find a feasible integral solution $\{\z, \ylor, \w, b\}$ subject to the constraints \cref{eqn:polyconst,eqn:disconst,eqn:conconst} combined}. Note that the problem formulation naturally handles partial states --- if $\left(\pstate_\ii\right)_j$ is $\top$ then $\left(\w_{\ci\di}\right)_j$ is set to zero.
 Equation~\eqref{eqn:polyconst} is applied over only the variables that don't map to $\top$. 
This ILP satisfies the following properties, the proofs of which are straightforward and presented below for completeness. We abuse notation by using 0 (resp. 1) and $\false$ (resp. $\true$) interchangeably.



\begin{theorem}\label{thm:correctness}%
  Any feasible solution to the ILP problem~\eqref{eqn:ILP} is a member of \intloopinv{}.
\end{theorem}
\begin{proof}
  Let $\{\z, \ylor, \w, b\}$ denote a feasible solution. First, note that for any fixed $c$, $\ylor_{\ii\ci} = 0$ iff $\sum_{\di \in [\nd]} \z_{\ii\ci\di} = 0$ because the conditions~\eqref{eqn:disconst} hold. As $\ylor_{\ii\ci} \in \{0,1\}$ and $\z_{\ii\ci\di} \in \{0,1\}$, it follows that $\ylor_{\ii\ci} = \underset{\di \in [\nd]}{\bigvee} \z_{\ii\ci\di}$, for each $c$. Next, it is immediate that $y_\ii = 1$ iff every $\ylor_{\ii\ci} = 1$ because the conditions~\eqref{eqn:conconst} hold. So, we have $y_\ii = \underset{\ci \in [\nc]}{\bigwedge} \ylor_{\ii\ci}$. Finally, notice that for any $\ii$, $\ci$, $\di$, $\z_{\ii\ci\di} = 1$ iff $\langle \w_{\ci\di}, \pstate_\ii \rangle + b_{\ci\di} > 0$ because the conditions~\eqref{eqn:polyconst} hold. In other words, $\z_{\ii\ci\di} = \ind{\langle \w_{\ci\di}, \pstate_\ii\rangle  + b_{\ci\di} > 0}$. Putting these together, we have, $y_\ii = \underset{\ci \in [\nc]}{\bigwedge} \ylor_{\ii\ci}  = \underset{\ci \in [\nc]}{\bigwedge} \underset{\di \in [\nd]}{\bigvee} \langle \w_{\ci\di}, \pstate_\ii\rangle  + b_{\ci\di} > 0 \in \intloopinv$. 
\end{proof}

\begin{theorem}
  Given a set of labeled program states $\{ \pstate_\ii, y_\ii \}, \ii \in [\nii]$, if there is a~$h \in \intloopinv{}$ s.t. $h(\pstate_\ii) = y_\ii$ for all $\ii \in [\nii]$, then the ILP problem~\eqref{eqn:ILP} has at least one feasible solution.
\end{theorem}
\begin{proof}
  This direction is easier to show. We can read off the integral coefficients $\w$ and $b$ for all the polynomials from $h$, and obtain $\z_{\ii\ci\di} = \ind{\langle \w_{\ci\di}, \pstate_\ii\rangle  + b_{\ci\di} > 0}$ so that~\eqref{eqn:polyconst} hold. Then assign $\ylor$ variables as in the proof of Claim~\ref{thm:correctness}, so that~\eqref{eqn:disconst} hold.  Finally, because $h(\pstate_\ii) = y_\ii$ holds for all $\ii \in [\nii]$, it follows that~\eqref{eqn:conconst} also hold. We have a feasible solution.
\end{proof}

Now, consider the problem of learning \emph{generalizable} predicates \inliststyle{(2)}.
To this end, we follow the Occam's razor principle -- seeking predicates that are ``simple''
and hence generalize better~\citep{beautiful}.
Simplicity in our case can be characterized by the size of the predicate clauses and the magnitude of the coefficients.
One way to achieve this is by constraining the $L_1$-norm of the coefficients $\w=[w_1,\ldots,w_n]$, \ie, by minimizing $\sum_{i\in[\ii]} |w_i|$.
Note that $L_1$-norm can be expressed using linear constraints:\footnote{%
  Integrality constraints on $\w^+, \w^-$ aren't needed, so the problem as stated is technically a mixed ILP.
} $\|\w\|_1 = \langle \mathbf{1}, \w^+ + \w^-\rangle$, where $\w^+ \geq 0$ and $\w^- \geq 0$ (componentwise inequality) such that $\w = \w^+ - \w^-$.

However, focusing only on the magnitude may lead to poor solutions.
For example, consider the Hoare triple: $\triple{n \geq 0 \wedge x = n \wedge y = 0}{\textbf{while} \texttt{ (x $>$ 0) } \textbf{do } \texttt{\{s $\gets$ y++; x--;\}}}{y = n}$.
Here, it is easy to verify that the loop invariant $x + y  = n$ is sufficient to assert \vc{post}.
The equivalent predicate in $\intloopinv$, $x + y - n \geq 0 \wedge n - x - y \geq 0$,
however has a larger $L_1$-norm though the invariant is a simple equality.
So, simply minimizing the $L_1$-norm is not sufficient.
Existing solvers~\cite{pie,piehe} employ heuristics such as preferring equality to inequality.
We handle this by explicitly penalizing the \emph{inclusion} of variables in the solution
by using a penalty $\mu$ where $\mu_j=0$ iff  $\forall \ci \in [\nc]. \forall \di \in [\nd].\left(\w_{\ci\di}\right)_j=0$. Intuitively, the more the number of variables with non-zero coefficients in the classifier, the  more the penalty.
Our final objective function combines both $\mu$ and $L_1$-norm  penalties:
\begin{align}\label{eqn:ILP}
  \min_{\w, \w^+, \w^-, b, \z, \ylor, \mathbf{\mu}} \sum_{\ci \in [\nc], \di \in [\nd]} \langle \mathbf{1}, \w^+_{\ci\di} + \w^-_{\ci\di}\rangle &\enskip+\enskip \lambda \langle \mathbf{1}, \mathbf{\mu} \rangle \nonumber\\
  \text{ subject to \cref{eqn:polyconst,eqn:disconst,eqn:conconst}}, \text{ and } \nonumber \\
  \mathbf{1} - M(\mathbf{1}-\mathbf{\mu}) \ \ \leq \  \sum_{\ci \in [\nc], \di \in [\nd]} \w^+_{\ci\di} + \w^-_{\ci\di} &\enskip\leq\enskip M\mathbf{\mu}, \nonumber \\ 
  \forall \ci \in [\nc], \di \in [\nd],\ \ \ \ \w_{\ci\di} &\enskip=\enskip \w^+_{\ci\di} - \w^-_{\ci\di}, \nonumber \\
  \w^+_{\ci\di} \geq 0\ ,\ \  \w^-_{\ci\di} \geq 0\ ,\ \  \mathbf{\mu} &\enskip\in\enskip \{0,1\}^{|\x|}.
\end{align}

The key advantage of the above ILP formulation is that it can be solved optimally by off-the-shelf solvers that leverage continuous and integer optimization techniques to solve such problems.
This enables efficient and scalable search compared to enumerative techniques.

We now formally give the implementation of {\sc \learn} procedure  in Figure~\ref{fig:learn}. {\sc GenerateExpression} (in Step 3) executed on the solution $\{\z, \ylor, \w, b\}$ to the ILP problem \cref{eqn:ILP} outputs the expression $\underset{\ci \in [\nc]}{\bigwedge} \underset{\di \in [\nd]}{\bigvee} \langle \w_{\ci\di}, \x\rangle  + b_{\ci\di} > 0$.

\begin{figure}[H]
\begin{algorithmic}[1]\linespread{1}\selectfont
\FunctionX{\textnormal{\sc \learn}}{$data$}
    \State $\{\z, \ylor, \w, b\} \gets$ Solve \cref{eqn:ILP} with labeled program states $\{\pstate_n, y_n\}_{n=1}^N$ from $data$
    \State $expr \gets \textnormal{\sc GenerateExpression}(\{\z, \ylor, \w, b\}) $ \Comment{returns a CNF expression over vars}
    \State \Return $expr$
\EndFunction
\end{algorithmic}
\caption{Implementation of {\sc Learn} procedure using the ILP formulation.}
\label{fig:learn}
\end{figure}

  In practice, it suffices to restrict $\w$ and $b$ to a small set of integers in \cref{eqn:polyconst}, and $M$ to be a very large integer.
  In the evaluation below, we use only two disjuncts\footnote{%
An equality requires two disjuncts: we flip the labels ($y_\ii$) in \cref{eqn:ILP}
and negate the optimal predicate.}, \ie, $\nc = 1$ and $\nd = 2$ in~\cref{eqn:ILP}, constrain the  coefficients to integers within $[-1000,1000]$   and use $M = 100,000$ in \cref{eqn:ILP}.
In \cref{fig:constrainteg}, we show the constraints generated by our ILP formulation for the Hoare triple $\triple{x = y = 0 }{\textbf{while} \texttt{ (x $\geq$ 0) } \textbf{do } \texttt{\{x $\gets$ x + y\}}}{False}$ from \citet{synergy}. We use the data in \cref{tab:constdata} to list our constraints. 
\input{figures/constraints}

%% file: figures/constraints.tex
\begin{figure}[H]
\begin{table}[H]
\begin{tabular}{ccl}
\toprule
$x$ & $y$ & $\ell$ \\ 
\midrule
0 & 0 & 1\\
-1 & $\top$ & 0\\
\bottomrule
\end{tabular}
\caption{Execution data.}
\label{tab:constdata}
\end{table}
\begin{minipage}{0.35\textwidth}
\centering
\begin{gather*}\label{eqn:constraints}
\text {Constraints for \cref{eqn:polyconst}}\\
-M(1 - \z_{111}) < 0*\w_{111} + 0*\w_{112} \leq \z_{111} \\
-M(1 - \z_{121}) < 0*\w_{122} + 0*\w_{122} \leq \z_{112} \\
-M(1 - \z_{111}) < -1*\w_{112}  \leq \z_{211} \\
-M(1 - \z_{121}) < -1*\w_{111}  \leq \z_{212} \\
\z_{111}, \z_{112}, \z_{211}, \z_{212} \in \{0,1\}\\
\text {Constraints for \cref{eqn:disconst}}\\
-M(1 - \ylor_{11}) < \z_{111} + \z_{112} \leq M\ylor_{11}\\
-M(1 - \ylor_{21}) < \z_{211} + \z_{212} \leq M\ylor_{21}\\
\ylor_{11}, \ylor_{21} \in \{0,1\}\\
\text {Constraints for \cref{eqn:conconst}}\\
\ylor_{11} \geq 1 \\
\ylor_{21} \leq 0 
\end{gather*}
\end{minipage}%
\hfill
\begin{minipage}{0.55\textwidth}
\centering
\begin{gather*}
\text {Constraints for \cref{eqn:ILP}}\\
\w_{111} =  \w^+_{111} + \w^-_{111}\\
\w_{112} =  \w^+_{112} + \w^-_{112}\\
\w_{121} =  \w^+_{121} + \w^-_{121}\\
\w_{122} =  \w^+_{122} + \w^-_{122} \\ 
1 - M(1 -\mathbf{\mu}_{1}) < \w^+_{111} + \w^-_{111} \leq M\mathbf{\mu}_{1}\\
1 - M(1 -\mathbf{\mu}_{1}) < \w^+_{121} + \w^-_{121} \leq M\mathbf{\mu}_{1}\\
1 - M(1 -\mathbf{\mu}_{2}) < \w^+_{112} + \w^-_{112} \leq M\mathbf{\mu}_{2}\\
1 - M(1 -\mathbf{\mu}_{2}) < \w^+_{122} + \w^-_{122} \leq M\mathbf{\mu}_{2}\\
\mathbf{\mu}_{1},\mathbf{\mu}_{2} \in \{0,1\}\\
\text{Objective function to minimize}\\
\min_{\w, \w^+, \w^-, \z, \ylor, \mathbf{\mu}} \w^+_{111} + \w^-_{111} + \w^+_{121} + \w^-_{121} \\
        + \w^+_{112} + \w^-_{112} + \w^+_{122} + \w^-_{122}  + \lambda(\mathbf{\mu}_{1} + \mathbf{\mu}_{2}) 
\end{gather*}
\end{minipage}
\caption{Example to show the constraints generated by our ILP formulation.}
\label{fig:constrainteg}
\end{figure}

%% file: sections/evaluation.tex
\section{Experimental Evaluation}%
\label{sec:evaluation}

We have implemented \oasis{} using the \lig~\citep{pie} framework in \sname{OCaml},
and using \sname{Z3}~\citep{z3} as the theorem prover for checking validity of the verification conditions.
We implement our technique for reducing the classification problem to ILPs in a \sname{Python} script,
which discharges the ILP subproblems to the \sname{OR-Tools}~\citep{web/or-tools} optimization package. \AlgoStyle{\FindPosCounterExample}\ and \AlgoStyle{\FindNegCounterExample}\ procedures in \cref{algo:refinement} are implemented in a python script and we use \sname{Z3}~\citep{z3} to solve the constraints.
We evaluate \oasis{} on commodity hardware --- a CPU-only machines with 2.5GHz Intel Xeon processor, 32\,GB RAM, and running Ubuntu Linux 18.04. 

\textbf{Solvers.}
We compare \oasis{}, against three tools:
\begin{andlist}
  \item \lig~\cite{pie} which uses data-driven invariant inference
  \item \cvc~\cite{cvc4,cegqi} which uses a refutation-based approach
  \item \dsynth~\cite{dryadsynth} which uses a combination of enumerative and deductive synthesis (cooperative synthesis).
\end{andlist}
\cvc and \lig are respectively the winners of the invariant-synthesis (Inv) track of \syguscomp['19]~\citep{syguscomp19} and \syguscomp['18]~\citep{syguscomp18}. Recently, \citet{dryadsynth} showed that their cooperative synthesis technique is able to perform better than \lig and \cvc on invariant synthesis tasks.

\input{figures/variable-stats}
\textbf{Benchmarks.}
We evaluate our technique on 403 instances which were part of the \syguscomp['19]~\citep{syguscomp19} and also studied by ~\citet{dryadsynth}. All these instances require reasoning over linear arithmetic. Out of these 403 instances, 276 were published by \citet{code2inv} and 127 were part of the \syguscomp['18]. The 276 instances are divided into three groups of 92 instances each \begin{andlist}
  \item \unconfounded{}
  \item \confoundedone{}
  \item \confoundedfive{}.
\end{andlist}
\confoundedone{} and \confoundedfive{} instances~\cite{code2inv} were obtained by adding irrelevant variables to each of the \unconfounded{} instances. The number of irrelevant variables ranges from 4-9 in \confoundedone{} and from 12-23 in \confoundedfive{}. In \cref{table:varstats}, we give key statistics of these benchmarks. These instances are provided as a collection of logic formulas representing the VCs (\Cref{sec:background.verification}) in the SyGuS grammar~\cite{sygusgrammar}.

\input{figures/sygus18-results}  

\subsection{Results on SyGuS Benchmarks }
\label{sec:res_sygus}
\textbf{Comparison with SyGuS Competitors. }%
We report the number of instances each tool solves with a timeout of 30 minutes\footnote{\citep{dryadsynth} uses a timeout of 30 minutes and we keep the same timeout.} in \cref{table:sygus_results}. \oasis{} synthesizes sufficient loop invariants on \textbf{353} instances, \textbf{7} more than the second best tool and \textbf{66} more than \cvc, the winner of invariant-synthesis (Inv) track of \syguscomp['19]. \oasis{} is able to solve \textbf{13} instances which no other tool can solve. In \cref{table:unqiue}, we list the number of unique benchmarks each tool solves. Out of the 353 instances that \oasis{} solves, 262 instances had disjunctive invariants. 

\input{figures/mayur-results}  
\textbf{Comparison with Data-Driven Tools. }%
\oasis solves 81 more benchmarks than \lig, which is the state-of-the-art data-driven invariant inference tool, the winner of \syguscomp['18]~\citep{syguscomp18},
and runner up of 
\syguscomp['19]~\citep{syguscomp19}. 
 In \cref{table:confound_results}, we give break down of the number of instances \lig and \oasis{} solves in each category of the 276 instances from~\citet{code2inv} to show how the complexity of benchmarks affects the performance of data-driven tools. \oasis{} solves 70 more benchmarks than \lig, indicating that \oasis{} scales better to programs with large number of variables. Recently, \citet{code2inv} (code2inv) and \citet{cln2inv} (cln2inv) propose  neural network based approaches for inferring invariants. We evaluate these two tools on \unconfounded{} instances\footnote{%
  \confoundedone{} and \confoundedfive{} instances  are only available as logic formulas and code2inv/cln2inv require C files as input. It is not straightforward to translate the constraints to C while maintaining a fair comparison.}, code2inv solves 64 and cln2inv solves 86 instances within 30 minutes.  \citet{icedt} and \citet{ddchc} are two other data-driven invariant inference tools.  From \cref{table:varstats}, the complexity of the 127 instances from \syguscomp['18] is lower than that of the 276 instances from~\citet{code2inv}. The 127 instances from \syguscomp['18]  subsumes the benchmarks these two tools were evaluated on and \oasis{} can solve all the instances that they succeeded on. 
  
  In \cref{fig:time_analysis}, we plot and compare the solving time for \oasis{} and \lig{} on the 403 instances. \oasis{} is slower that \lig on most benchmarks because
  it runs many \sname{Z3} queries to refine the set of relevant variables.
 Identifying relevant variables helps \oasis{} scale to programs with large number of variables and \oasis{} solves fewer instances without it (\Cref{sec:ablation}).  
\input{figures/time}

\input{figures/statistics-table}
\textbf{Analysis.}
In \cref{table:xc_stats}, we give detailed statistics of 120 instances which are \confoundedone{} and \confoundedfive{} instances that \oasis solves\footnote{%
  We don't include the instances where the post-condition is a sufficient invariant.}. \oasis{} takes 183.77 seconds on average for invariant synthesis. \lig times out on 51 instances and averages 15.36 seconds on the 69 instances it solves. We give the total number of variables in the instance, the number of relevant variables \oasis{} uses to synthesize the invariant and the number of variables in the gold solution, \ie, the number of variables appearing in the invariant of the corresponding \unconfounded{} instance. \oasis{} reduces the number of variables it uses to solve the problem by $3\times$ on average. Size of the invariant is computed as the number of nodes in the SyGuS AST~\cite{sygusgrammar} of the invariant. 

\subsection{Ablation Study}
\label{sec:ablation}

In \Cref{sec:res_sygus}, we saw that \oasis{} solves more instances than any other tool. Now, through this study, we try to answer the following questions about \oasis:
\begin{enumerate}
    \item \textit{Does identifying relevant variables really help?}
    \item \textit{Is our ILP formulation better than other techniques?}
    \item \textit{Is our relevant variable identifying algorithm by itself sufficient to guess invariants?}
\end{enumerate}

\input{figures/ablation-active}

To show that inference of relevant variables helps in solving more instances, we compare our tool, \oasis{}, against the following configuration:
\begin{enumerate}
\item \textbf{\oasiscomp{}.} We don't use our relevant variables identifying algorithm, \ie, we don't execute lines 1-3, thread 1 and thread 2 in \cref{algo:oasis}. We use all the variables appearing in the program to infer the invariant. We still use our ILP formulation (\Cref{sec:learning}) for the \AlgoStyle{\LearnClassifier}\ function in \cref{algo:relevant-infer}.
\end{enumerate}
We observe from \cref{table:ablation_act} that 
inferring relevant variables helps solve 27  more instances.  Moreover, these results also show that out of the 81 benchmarks that \oasis solves more than \lig, 54 are because of replacing the exhaustive enumeration-based feature synthesizer in \lig by ILP and 27 are because of the ILP-based relevant variable inference. 

\input{figures/ablation-learner}

Next, to show that our ILP formulation is better suited for inferring invariants, we compare \oasis{} to the following configurations:

\input{figures/naive}
\begin{enumerate}
\item \label{itm:first} \textbf{\oasiscl{}.}  We use an exhaustive enumeration strategy to find the set of relevant variables instead of using our ILP formulation (\Cref{sec:learning}) to identify this set. We replace our \cref{algo:oasis} with the implementation in \cref{fig:naive}. 
\item \label{itm:second} \textbf{\oasisdt{}.} We use scikit-learn \cite{web/dt-scikit} implementation of decision tree in place of our ILP formulation (\Cref{sec:learning}) for the \AlgoStyle{\LearnClassifier}\ function in \cref{algo:oasis} to find the set of relevant variables. There is no constraint on the height of the tree that the decision tree algorithm can learn.
\item \label{itm:third} \textbf{\oasiscomplete{}.} Instead of using the partial maps returned by \sname{Z3} while sampling states in \cref{algo:refinement}, we use complete maps. Partial maps are completed by replacing the don't care values with random integers.
\item \label{itm:fourth} \textbf{\oasisopt{}.}  We ignore the objective function used in \cref{eqn:ILP}, which biases our learner towards  simple  classifiers, and instead only search for solutions that satisfy the constraints generated by our ILP formulation. This learner with no optimization is used for the \AlgoStyle{\LearnClassifier}\ function in both \cref{algo:oasis} and \cref{algo:relevant-infer}.
\end{enumerate}
    Again, in the configurations \ref{itm:first}, \ref{itm:second} and \ref{itm:third}, we still use our ILP formulation (\Cref{sec:learning}) for the \AlgoStyle{\LearnClassifier}\ function in \cref{algo:relevant-infer}.\\

We observe from \cref{table:ablation_learn} that using a simple strategy like enumerating over combinations of variables is not adequate to find the set of relevant variables. Also, this strategy is not scalable for problems with a large number of variables. Decision trees have been widely used for classification tasks \cite{ddchc,icedt}. Our ILP formulation performs slightly better than decision trees because of the objective function which has specific penalties to learn simple and generalizable expressions. We also observe that running our ILP formulation without the objective function results in expressions without any constraints on the coefficients of the variables and the size of expression and we solve far less number of instances than running with the objective function. This shows that solving only the search problem in our ILP formulation is insufficient for generalization and learning expressions consistent with the Occam's razor principle helps in solving more instances.
 \input{figures/ablation-lig}
Finally, one might think that the classifier that separates reachable and bad states during the inference of relevant variables might be a good guess for a sufficient loop invariant. To show that this classifier is usually not an invariant, we compare \oasis{} against the following configuration:
\begin{enumerate}
\item \textbf{\oasisact{}.} We don't run \AlgoStyle{\RelInfer}, \ie, thread 3 in \cref{algo:oasis}. We use the classifier \Classifier{} returned by the \AlgoStyle{\LearnClassifier}\ function in \cref{algo:oasis} as our invariant guess. 
\end{enumerate}
From \cref{table:ablation_lig}, we observe that the classifier learnt between the good and bad states in \cref{algo:oasis} is not usually a sufficient invariant. However, it is still an indicator of the relevant variables which appear in the sufficient invariant. 

%% file: figures/variable-stats.tex
\begin{table}[h!]
\centering
 \begin{tabular}{|c c c c c|} 
 
 \hline
 \textbf{\relscale{1.1} Benchmark} & \multicolumn{3}{c}{\makecell{\bf \# Variables}}  & \makecell{\bf \# Instances }\\ 
 \hline\hline
 & \textbf{Median} & \textbf{Average} & \textbf{Maximum} & \\
\hline
 \sygus & 3 & 4  & 9 & 127 \\
  \hline
 \unconfounded & 10 & 11 & 22 & 92 \\ 
 \hline
 \confoundedone & 15 & 16 & 32 & 92\\
 \hline
 \confoundedfive & 25 & 26 & 44 & 92\\
 \hline

 \end{tabular}
 \caption{Statistics of the 403 SyGuS instances used for evaluation.}
\label{table:varstats}
\end{table}

%% file: figures/sygus18-results.tex
\begin{minipage}{.5\textwidth}
\begin{table}[H]
\begin{tabular}{|c c|} 
 
 \hline
 \textbf{\relscale{1.1} Tool} & \makecell{\bf Solved \\ \smaller{\bf (out of 403)}}  \\ 
 \hline\hline
 \cvc & 287 \\ 
 \hline
 \dsynth & 346 \\
 \hline
 \lig & 272  \\
 \hline
 \oasis & \textbf{353} \\  
 \hline
 \end{tabular}

\caption{Comparison of \oasis{} with SyGuS tools on the 403 instances which were part of the \syguscomp['19]~\citep{syguscomp19} invariant synthesis track.}
\label{table:sygus_results}
\end{table}
\end{minipage}
\hfill
\begin{minipage}{.4\textwidth}
\begin{table}[H]
\centering
 \begin{tabular}{|c c|} 
 
 \hline
 \textbf{\relscale{1.1} Tool} & \makecell{\bf Solved} \\ 
 \hline\hline
  \cvc & 0\\
  \hline
 \dsynth & 11 \\ 
 \hline
 \lig & 1\\
 \hline
 \oasis & \textbf{13}\\
 \hline

 \end{tabular}
 \caption{Number of uniquely solved instances by each tool.}
\label{table:unqiue}
\end{table}
\end{minipage}

%% file: figures/mayur-results.tex
\begin{table}[h!]
\centering
\resizebox{\textwidth}{!} {
 \begin{tabular}{|c c c c c|} 
 
 \hline
 \textbf{\relscale{1.1} Tool} & \makecell{\bf \unconfounded \\ \smaller{\bf (out of 92)}} & \makecell{\bf \confoundedone \\ \smaller{\bf (out of 92)}} & \makecell{\bf \confoundedfive \\ \smaller{\bf (out of 92)}} & \makecell{\bf Solved \\ \smaller{\bf (out of 276)}} \\ 
 \hline\hline
 \lig & 62 & 59 & 44 & 165 \\
 \hline
 \oasis & 84 & 84 & 67 & \textbf{235} \\  
 \hline
 \end{tabular}
 }
 \caption{Comparison of \oasis{} and \lig, the state-of-the-art data driven tool, on the 276 instances which were part of the \syguscomp['19]~\citep{syguscomp19} and studied by \citet{code2inv}. These 276 instance are more complex than the remaining 127 instances in terms of the number of variables present in them. These results indicate \oasis{} scales than better \lig on program with large number of variables.} 
 
\label{table:confound_results}
\vspace{-20pt}
\end{table}

%% file: figures/time.tex
\begin{figure}[H]
\includegraphics[width=\textwidth,height=\textheight,keepaspectratio, trim=0 8.5cm 0 9cm, clip]{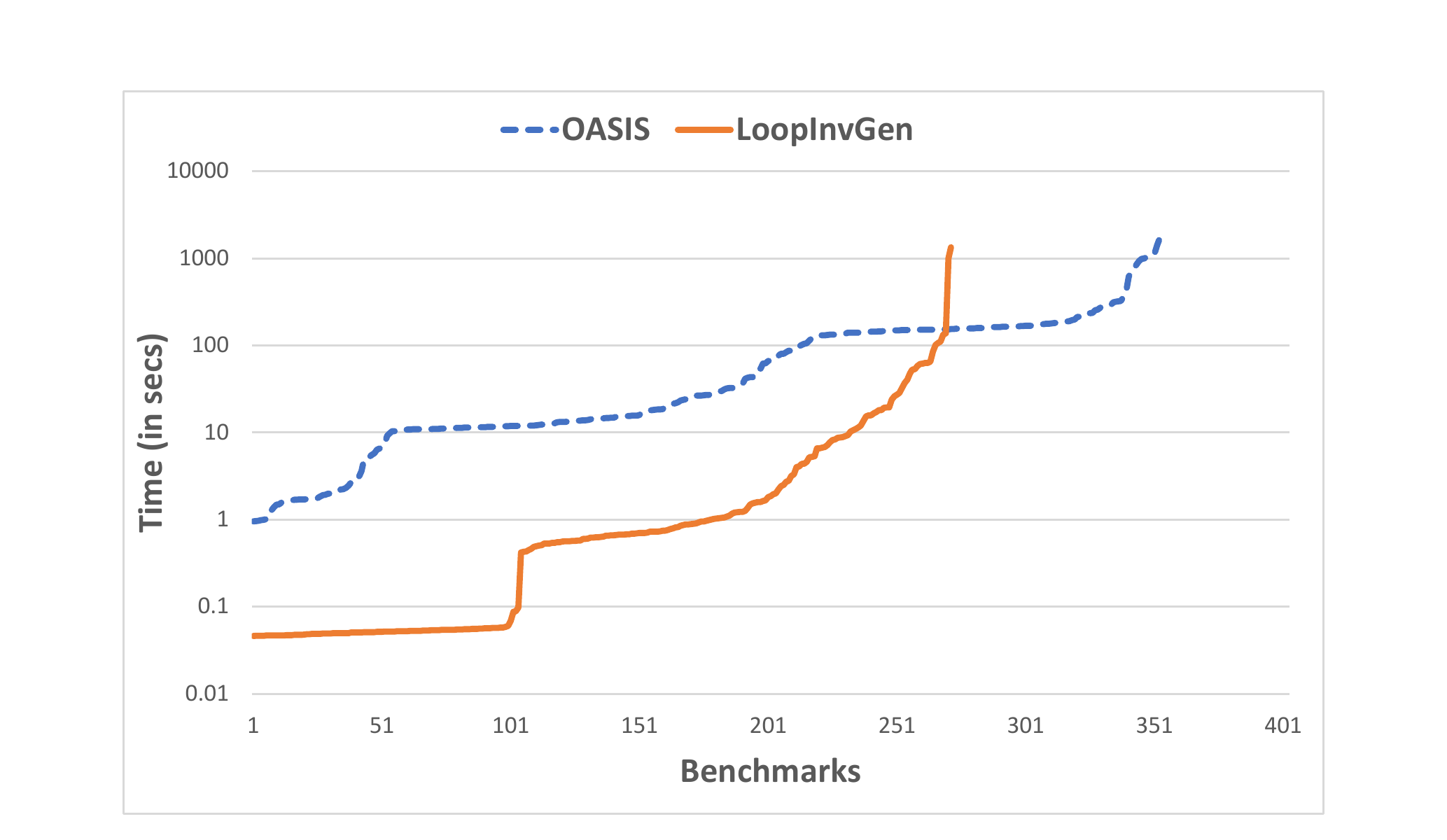}
\caption{Solving time comparison of \oasis{} and \lig on the 403 SyGuS instances.}
\label{fig:time_analysis}
\end{figure}

%% file: figures/statistics-table.tex
\begin{scriptsize}

 \begin{longtable}{|>{\centering\arraybackslash}p{1.75cm} >{\raggedleft\arraybackslash}p{1.65cm}  >{\raggedleft\arraybackslash}p{1.65cm} >{\raggedleft\arraybackslash}p{1.65cm} >{\raggedleft\arraybackslash}p{1.75cm} >{\raggedleft\arraybackslash}p{1.65cm} >{\raggedleft\arraybackslash}p{1.25cm}|}
 \caption{Details of \oasis{} invariant synthesis time, \lig invariant synthesis time, total variables in the instance,  number of relevant variable used by \oasis{} in the successful run of \AlgoStyle{\RelInfer}, number of variables in the invariant of the corresponding \unconfounded{} instance (Gold Solution) and size of the synthesized invariant for \confoundedone{} and \confoundedfive{} instances solved by \oasis{}. A `-' indicates that  \lig times out on that instance. }\label{table:xc_stats}\\
 \hline
 
 Benchmark &  \oasis{} Time &  \lig Time & \# Variables & \# Relevant Variables & Gold Solution & Size \\
 \hline\hline
1\textunderscore conf1 & 32.46 & - & 15 & 10 & 2 & 28 \\
\hline
2\textunderscore conf1 & 20.83 & - & 15 & 6 & 2 & 28 \\
\hline
3\textunderscore conf1 & 142.35 & - & 18 & 9 & 3 & 11 \\
\hline
5\textunderscore conf1 & 747.97 & - & 20 & 11 & 3 & 11 \\
\hline
10\textunderscore conf1 & 54.31 & 0.68 & 14 & 6 & 1 & 24 \\
\hline
11\textunderscore conf1 & 13.8 & 0.75 & 14 & 6 & 1 & 24 \\
\hline
12\textunderscore conf1 & 41.67 & 0.69 & 14 & 6 & 1 & 24 \\
\hline
13\textunderscore conf1 & 15.58 & 0.73 & 14 & 6 & 1 & 24 \\
\hline
15\textunderscore conf1 & 156.51 & - & 20 & 9 & 2 & 32 \\
\hline
16\textunderscore conf1 & 191.54 & - & 20 & 9 & 2 & 7 \\
\hline
17\textunderscore conf1 & 322.36 & - & 20 & 14 & 2 & 32 \\
\hline
18\textunderscore conf1 & 364.25 & - & 20 & 9 & 2 & 7 \\
\hline
23\textunderscore conf1 & 76.15 & - & 15 & 6 & 2 & 60 \\
\hline
24\textunderscore conf1 & 26.3 & - & 15 & 6 & 2 & 60 \\
\hline
28\textunderscore conf1 & 67.35 & - & 11 & 5 & 2 & 7 \\
\hline
29\textunderscore conf1 & 272.63 & - & 11 & 8 & 2 & 11 \\
\hline
38\textunderscore conf1 & 140.04 & 12.01 & 16 & 7 & 1 & 3 \\
\hline
40\textunderscore conf1 & 158.92 & 0.97 & 16 & 6 & 1 & 3 \\
\hline
41\textunderscore conf1 & 159.94 & 0.96 & 16 & 7 & 2 & 19 \\
\hline
42\textunderscore conf1 & 620.74 & 11.46 & 16 & 6 & 1 & 3 \\
\hline
43\textunderscore conf1 & 158.4 & 4.39 & 16 & 5 & 1 & 3 \\
\hline
44\textunderscore conf1 & 157.74 & 5.23 & 16 & 5 & 1 & 3 \\
\hline
45\textunderscore conf1 & 163.14 & 18.16 & 16 & 7 & 1 & 3 \\
\hline
46\textunderscore conf1 & 272.05 & - & 16 & 7 & 2 & 18 \\
\hline
47\textunderscore conf1 & 159.41 & 19.41 & 16 & 7 & 1 & 3 \\
\hline
48\textunderscore conf1 & 158 & 11.01 & 16 & 7 & 1 & 3 \\
\hline
49\textunderscore conf1 & 158.05 & 15.35 & 16 & 7 & 1 & 3 \\
\hline
56\textunderscore conf1 & 234.26 & - & 16 & 6 & 2 & 7 \\
\hline
57\textunderscore conf1 & 168.95 & 0.82 & 16 & 6 & 2 & 7 \\
\hline
65\textunderscore conf1 & 11.46 & 0.05 & 14 & 6 & 2 & 22 \\
\hline
71\textunderscore conf1 & 140.38 & 110.69 & 21 & 4 & 1 & 3 \\
\hline
77\textunderscore conf1 & 165.67 & 0.75 & 16 & 8 & 2 & 7 \\
\hline
78\textunderscore conf1 & 141.53 & 37.24 & 16 & 6 & 1 & 3 \\
\hline
79\textunderscore conf1 & 140.24 & 32.57 & 16 & 6 & 1 & 3 \\
\hline
91\textunderscore conf1 & 14.82 & - & 12 & 5 & 2 & 7 \\
\hline
94\textunderscore conf1 & 80.46 & - & 19 & 10 & 4 & 14 \\
\hline
95\textunderscore conf1 & 177.43 & 1.36 & 20 & 10 & 3 & 53 \\
\hline
96\textunderscore conf1 & 186.7 & 0.89 & 20 & 10 & 2 & 53 \\
\hline
97\textunderscore conf1 & 15.91 & 0.69 & 20 & 5 & 1 & 3 \\
\hline
98\textunderscore conf1 & 14.77 & 0.68 & 20 & 5 & 1 & 3 \\
\hline
99\textunderscore conf1 & 86.83 & 1.86 & 17 & 8 & 3 & 39 \\
\hline
100\textunderscore conf1 & 26.67 & - & 17 & 8 & 3 & 30 \\
\hline
103\textunderscore conf1 & 10.91 & 0.05 & 9 & 3 & 1 & 19 \\
\hline
107\textunderscore conf1 & 281.15 & - & 21 & 9 & 3 & 11 \\
\hline
108\textunderscore conf1 & 140.97 & 1.08 & 23 & 5 & 2 & 7 \\
\hline
109\textunderscore conf1 & 324.32 & - & 23 & 11 & 3 & 11 \\
\hline
110\textunderscore conf1 & 54.29 & 8.79 & 17 & 8 & 2 & 63 \\
\hline
111\textunderscore conf1 & 144.44 & 9.14 & 17 & 8 & 2 & 63 \\
\hline
114\textunderscore conf1 & 13.79 & 1.06 & 16 & 6 & 1 & 19 \\
\hline
115\textunderscore conf1 & 15.71 & 0.86 & 16 & 6 & 1 & 19 \\
\hline
118\textunderscore conf1 & 15.71 & 8.75 & 17 & 8 & 2 & 63 \\
\hline
119\textunderscore conf1 & 18.63 & 8.86 & 17 & 8 & 2 & 63 \\
\hline
120\textunderscore conf1 & 21.32 & - & 15 & 6 & 2 & 56 \\
\hline
121\textunderscore conf1 & 25.12 & - & 15 & 6 & 2 & 56 \\
\hline
124\textunderscore conf1 & 101.6 & - & 19 & 10 & 4 & 47 \\
\hline
125\textunderscore conf1 & 32.25 & - & 19 & 10 & 4 & 47 \\
\hline
130\textunderscore conf1 & 114.27 & - & 32 & 12 & 3 & 11 \\
\hline
131\textunderscore conf1 & 104.58 & - & 32 & 12 & 3 & 11 \\
\hline
132\textunderscore conf1 & 1087.6 & 1.59 & 26 & 14 & 1 & 3 \\
\hline
1\textunderscore conf5 & 62.36 & - & 24 & 9 & 2 & 14 \\
\hline
2\textunderscore conf5 & 25.22 & - & 24 & 6 & 2 & 14 \\
\hline
10\textunderscore conf5 & 23.32 & 1.97 & 23 & 6 & 1 & 7 \\
\hline
11\textunderscore conf5 & 27.18 & 1.82 & 23 & 6 & 1 & 7 \\
\hline
12\textunderscore conf5 & 73.78 & 2.52 & 23 & 6 & 1 & 7 \\
\hline
13\textunderscore conf5 & 27.07 & 2.78 & 22 & 6 & 1 & 7 \\
\hline
16\textunderscore conf5 & 1131.16 & - & 29 & 9 & 2 & 7 \\
\hline
18\textunderscore conf5 & 215.43 & - & 29 & 9 & 2 & 7 \\
\hline
24\textunderscore conf5 & 849.02 & - & 24 & 8 & 2 & 41 \\
\hline
25\textunderscore conf5 & 14.68 & - & 17 & 3 & 1 & 3 \\
\hline
28\textunderscore conf5 & 1001.83 & - & 19 & 10 & 2 & 7 \\
\hline
30\textunderscore conf5 & 13.48 & - & 17 & 3 & 1 & 3 \\
\hline
36\textunderscore conf5 & 189.78 & 1.03 & 24 & 6 & 1 & 3 \\
\hline
38\textunderscore conf5 & 152.22 & 28.63 & 26 & 7 & 1 & 3 \\
\hline
40\textunderscore conf5 & 166.24 & 15.81 & 26 & 6 & 1 & 3 \\
\hline
41\textunderscore conf5 & 174.61 & 2.74 & 26 & 7 & 2 & 7 \\
\hline
42\textunderscore conf5 & 165.92 & 40.38 & 26 & 6 & 1 & 3 \\
\hline
43\textunderscore conf5 & 163.34 & 13.57 & 26 & 5 & 1 & 3 \\
\hline
44\textunderscore conf5 & 180.09 & 58.54 & 24 & 5 & 1 & 3 \\
\hline
45\textunderscore conf5 & 171.74 & 62.23 & 26 & 7 & 1 & 3 \\
\hline
46\textunderscore conf5 & 178.62 & - & 26 & 7 & 2 & 27 \\
\hline
47\textunderscore conf5 & 167.27 & 52.93 & 26 & 7 & 1 & 3 \\
\hline
48\textunderscore conf5 & 169.15 & 63.3 & 26 & 7 & 1 & 3 \\
\hline
49\textunderscore conf5 & 165.85 & 53.48 & 24 & 7 & 1 & 3 \\
\hline
50\textunderscore conf5 & 164.16 & 1.54 & 24 & 5 & 1 & 3 \\
\hline
51\textunderscore conf5 & 199.07 & 1.23 & 24 & 7 & 1 & 3 \\
\hline
56\textunderscore conf5 & 177.1 & 26.27 & 26 & 6 & 2 & 7 \\
\hline
57\textunderscore conf5 & 700.96 & 1.49 & 24 & 6 & 2 & 7 \\
\hline
63\textunderscore conf5 & 90.45 & - & 23 & 6 & 2 & 7 \\
\hline
64\textunderscore conf5 & 84.12 & - & 23 & 8 & 2 & 7 \\
\hline
65\textunderscore conf5 & 122.15 & - & 23 & 9 & 2 & 7 \\
\hline
67\textunderscore conf5 & 88 & - & 25 & 8 & 2 & 7 \\
\hline
68\textunderscore conf5 & 289.74 & - & 25 & 8 & 3 & 14 \\
\hline
70\textunderscore conf5 & 122.22 & - & 25 & 8 & 3 & 11 \\
\hline
71\textunderscore conf5 & 146.19 & 61.43 & 29 & 4 & 1 & 3 \\
\hline
77\textunderscore conf5 & 171.97 & 1.23 & 24 & 8 & 2 & 7 \\
\hline
78\textunderscore conf5 & 151.23 & 102.33 & 24 & 6 & 1 & 3 \\
\hline
79\textunderscore conf5 & 183.15 & - & 24 & 6 & 1 & 3 \\
\hline
83\textunderscore conf5 & 236.27 & - & 23 & 6 & 2 & 7 \\
\hline
84\textunderscore conf5 & 32.48 & - & 23 & 6 & 2 & 7 \\
\hline
85\textunderscore conf5 & 162.42 & - & 23 & 6 & 2 & 7 \\
\hline
91\textunderscore conf5 & 24.29 & 1.02 & 20 & 5 & 2 & 7 \\
\hline
94\textunderscore conf5 & 185.83 & - & 27 & 10 & 4 & 17 \\
\hline
95\textunderscore conf5 & 259.72 & 24.03 & 29 & 10 & 3 & 53 \\
\hline
96\textunderscore conf5 & 215.47 & 47.57 & 29 & 10 & 3 & 53 \\
\hline
97\textunderscore conf5 & 35.41 & 0.8 & 29 & 5 & 1 & 3 \\
\hline
98\textunderscore conf5 & 26.88 & 0.7 & 29 & 5 & 1 & 3 \\
\hline
99\textunderscore conf5 & 42.8 & 2.01 & 26 & 8 & 3 & 39 \\
\hline
100\textunderscore conf5 & 180.4 & - & 26 & 8 & 3 & 27 \\
\hline
101\textunderscore conf5 & 147.21 & - & 19 & 5 & 2 & 11 \\
\hline
102\textunderscore conf5 & 106.3 & - & 19 & 5 & 2 & 11 \\
\hline
103\textunderscore conf5 & 18.53 & 0.95 & 17 & 6 & 1 & 3 \\
\hline
107\textunderscore conf5 & 162.78 & - & 30 & 9 & 3 & 11 \\
\hline
108\textunderscore conf5 & 149.05 & 3.31 & 32 & 5 & 2 & 7 \\
\hline
114\textunderscore conf5 & 23.81 & 1.05 & 25 & 6 & 1 & 31 \\
\hline
115\textunderscore conf5 & 22.33 & 0.93 & 25 & 6 & 1 & 31 \\
\hline
120\textunderscore conf5 & 66.52 & - & 24 & 6 & 2 & 37 \\
\hline
124\textunderscore conf5 & 1068.72 & - & 28 & 10 & 4 & 47 \\
\hline
128\textunderscore conf5 & 21.2 & 18.11 & 19 & 5 & 1 & 3 \\
\hline
132\textunderscore conf5 & 1681.38 & 5.35 & 38 & 14 & 1 & 3 \\
\hline
133\textunderscore conf5 & 17.97 & 27.28 & 19 & 5 & 2 & 7 \\
\hline

\end{longtable}

\end{scriptsize}

%% file: figures/ablation-active.tex
\begin{table}[h!]
\centering
{
 \begin{tabular}{|c c|} 
 
 \hline
 \textbf{\relscale{1.1} Tool} & \makecell{\bf Solved \\ \smaller{\bf (out of 403)}}  \\ 
 \hline\hline
 \oasiscomp & 326 \\
 \hline
 \oasis & \textbf{353}  \\  
 \hline
 \end{tabular}
 }
 \caption{Comparison of \oasis{} and \oasis{} without our relevant variable identifying algorithm. \oasiscomp uses all the variables appearing in the program as relevant variables.}
\label{table:ablation_act}
\end{table}

%% file: figures/ablation-learner.tex
\begin{table}[h!]
\centering
{
 \begin{tabular}{|c c|} 
 
 \hline
 \textbf{\relscale{1.1} Tool} & \makecell{\bf Solved \\ \smaller{\bf (out of 403)}}  \\ 
 \hline\hline
 \oasiscl & 258   \\  
 \hline
 \oasisdt & 338  \\  
 \hline
 \oasiscomplete & 336  \\  
 \hline
 \oasisopt & 297   \\  
 \hline
 \oasis & \textbf{353}  \\  
 \hline
 \end{tabular}
 }
 \caption{Comparison of \oasis{} and \oasis{} with our ILP formulation replaced with naive enumeration, decision tree, ILP without objective function and ILP with complete maps instead of partial maps.}
\label{table:ablation_learn}
\end{table}

%% file: figures/naive.tex
\begin{figure}
\begin{algorithmic}[1]
        \small\linespread{1}\selectfont
    \FunctionX{\oasiscl}{$\langle \PreRel,\TransRel,\PostRel \rangle$ : Verification Problem, $\PosPStates$ : States, $\NegPStates$ : States, $\textnormal{Vars}$: Variables in Problem}
     \For{\textbf{each} subset $s$ of $\textnormal{\sc Vars}$}
     
        \State Variables $\vec{r} \gets s$
        \State $\InvRel \gets \Call{\RelInfer}{\langle \PreRel,\TransRel,\PostRel \rangle, \PosPStates, \NegPStates, \vec{r}} \, \big\vert_{\,\textsf{timeout = }\tau}$
        \IfThen {$\InvRel \neq \bot$}
                \Return $\InvRel$
     \EndFor
    \EndFunction
\end{algorithmic}
\caption{Implementation of \oasis{} with naive enumeration strategy. $\vec{r}$ is the set of relevant variables.}
\label{fig:naive}
\end{figure}

%% file: figures/ablation-lig.tex
\begin{table}[H]
\centering
{
 \begin{tabular}{|c c|} 
 
 \hline
 \textbf{\relscale{1.1} Tool} & \makecell{\bf Solved \\ \smaller{\bf (out of 403)}}  \\ 
 \hline\hline
 \oasisact & 82 \\
 \hline
 \oasis & \textbf{353}  \\  
 \hline
 \end{tabular}
 }
 \caption{Comparison of \oasis{} and \oasis{} without our relevance-aware invariant inference algorithm (thread 3) in \cref{algo:oasis}.}
\label{table:ablation_lig}
\end{table}

%% file: sections/related-work.tex
\section{Related Work}%
\label{sec:related}
Loop invariant inference is a challenging problem with a long history. Although, we focus on numerical invariants in this paper, invariant inference over practical programs can be reduced to numerical reasoning~\cite{slam}. The existing techniques for numerical loop invariants can be classified in two categories: those that are purely static and infer invariants from program text and data-driven approaches that guess invariants from examples of program states. The traditional static approaches for inferring loop invariants include
abstract interpretation~\citep{astree,cousot77},
predicate abstraction~\citep{slam,yogi}, interpolation~\citep{blast,jm06},
constraint solving~\citep{sting},
and abductive inference~\citep{hola,abduct}.
Although these approaches are mature and can scale to large programs, the data-driven approaches are more recent and the scalability is currently limited. However, data-driven invariant inference techniques (\eg,~\citet{ice,pie}) have been shown to outperform static approaches for verification of small but non-trivial loops. 

\oasis reduces the problem of invariant inference to solving a constrained ILP problem, where a solver minimizes a penalty while maintaining the feasibility of constraints. 
Similar to us,~\cite{ssd,sting,pldi08} also reduce invariant inference to constraint solving. However, their constraints are non-linear and much harder to solve. Subsequently,~\citet{invgen} use data to make these constraints linear. However, this line of work either doesn't support disjunctive invariants or requires the number of disjunctions to be fixed by a user-provided template. \oasis has no such restrictions and can generate invariants that are arbitrary Boolean combinations of linear inequalities.
Moreover, these techniques only solve the feasibility problem and do not have penalty terms. Although we can encode the search component of our ILP problem as an SMT constraint~\cite{ice}, our penalty terms are effective at generalization (\cref{table:ablation_learn}). Thus, we use an optimization framework instead of a constraint solving framework like prior works. 

Unlike ~\cite{vueq,dig,esop,geometric}, which are data-driven techniques for non-linear invariants, \oasis focuses on linear invariants. This design is primarily motivated by the presence of hundreds of benchmarks from the SyGuS competition. There are only a few benchmarks for non-linear invariants and these can already be solved well by  existing techniques~\cite{dig,yaononlin}. Additionally, since \oasis{} is implemented on top of \lig, it inherits the capabilities to infer invariants for multiple loops and nested loops from~\citet{pie}. 

Prior work on data-driven techniques to infer arbitrary Boolean combinations of linear inequalities have all been evaluated on  benchmarks at the scale of the SyGuS'18 benchmarks (\cref{table:varstats}) that have less than ten variables. These include techniques that use SMT solvers directly~\cite{ice}, PAC-learning~\cite{geometric}, decision trees~\cite{icedt}, SVMs~\cite{selective}, and combinations of SVMs and decision trees~\cite{ddchc}.
Techniques based on  
neural networks~\cite{code2inv,cln2inv} have been evaluated at the scale of Unconfounded benchmarks (\cref{table:varstats}). \oasis uses ILP to scale data-driven inference to succeed on benchmarks with even more variables. 

%% file: sections/conclusion.tex
\section{Conclusions}%
\label{sec:conclusions}
\oasis makes the following contributions. Conceptually,
\oasis reduces the problem of invariant inference to learning relevant variables and learning features. Technically, \oasis provides a novel ILP-based learner which generates sparse classifiers and solves both these problems effectively. Practically, \oasis outperforms the state-of-the-art tools, including the most recent work of~\citet{dryadsynth}, on benchmarks from the invariant inference track of the Syntax Guided Synthesis competition. \oasis both solves more benchmarks and can solve benchmarks that no other tool could solve before. We are  working towards integrating \oasis with a full-fledged verification system for effective verification of complete applications.


Stepping back, the inference of loop invariants is an old problem with a rich history. Many techniques have been applied to this problem and they all have their strengths and weaknesses. Data-driven invariant inference techniques can handle challenging loops with confusing program text by applying ML techniques to mine patterns directly from data. However, these techniques have been evaluated only on loops with a small number of variables. This weakness is clear on benchmarks with large number of irrelevant variables. \oasis  uses ML  to infer the relevant variables which leads to simpler verification problems with fewer variables. We believe that this idea of simplifying the verification problems using ML is generally applicable. 
\oasis demonstrates that  ML-based simplification is effective for data-driven invariant inference and we will explore it in other contexts in the future.
